\documentclass[runningheads]{llncs}

\usepackage{eccv}

\usepackage{eccvabbrv}

\usepackage{graphicx}
\usepackage{booktabs}
\usepackage{amsmath, bm}
\usepackage[dvipsnames]{xcolor}
\usepackage{multirow}
\usepackage{multicol}
\usepackage{makecell}
\usepackage{pifont}
\usepackage[accsupp]{axessibility}  

\usepackage[pagebackref,breaklinks,colorlinks]{hyperref}
\makeatletter
\newcommand{\printfnsymbol}[1]{%
  \textsuperscript{\@fnsymbol{#1}}%
}
\makeatother
\newtheorem{assumption}{Assumption}

\DeclareMathOperator*{\argmin}{arg\,min}
\usepackage{orcidlink}
\begin{document}

\title{\texttt{$\infty$-Brush} \includegraphics[height=11pt,width=11pt]{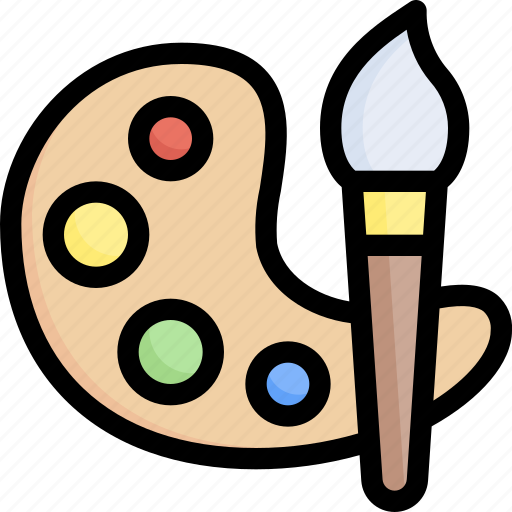}: Controllable Large Image Synthesis with Diffusion Models in Infinite Dimensions} 
\titlerunning{\texttt{$\infty$-Brush}}

\author{Minh-Quan Le\thanks{Equal contribution}\and
Alexandros Graikos\printfnsymbol{1}\and
Srikar Yellapragada\and \\Rajarsi Gupta \and Joel Saltz \and Dimitris Samaras}

\authorrunning{M.-Q.~Le et al.}

\institute{Stony Brook University\\
\email{\{mile,agraikos,myellapragad,samaras\}@cs.stonybrook.edu}}

\maketitle

\begin{abstract}
Synthesizing high-resolution images from intricate, domain-specific information remains a significant challenge in generative modeling, particularly for applications in large-image domains such as digital histopathology and remote sensing. Existing methods face critical limitations: conditional diffusion models in pixel or latent space cannot exceed the resolution on which they were trained without losing fidelity, and computational demands increase significantly for larger image sizes. Patch-based methods offer computational efficiency but fail to capture long-range spatial relationships due to their overreliance on local information. In this paper, we introduce a novel conditional diffusion model in infinite dimensions, \texttt{$\infty$-Brush} for controllable large image synthesis. We propose a cross-attention neural operator to enable conditioning in function space. Our model overcomes the constraints of traditional finite-dimensional diffusion models and patch-based methods, offering scalability and superior capability in preserving global image structures while maintaining fine details. To our best knowledge, \texttt{$\infty$-Brush} is the first conditional diffusion model in function space, that can controllably synthesize images at arbitrary resolutions of up to $4096\times4096$ pixels. The code is available at \href{https://github.com/cvlab-stonybrook/infinity-brush}{https://github.com/cvlab-stonybrook/infinity-brush}.

\end{abstract}

\begin{figure}[!h]
    \centering
    \includegraphics[width=\linewidth]{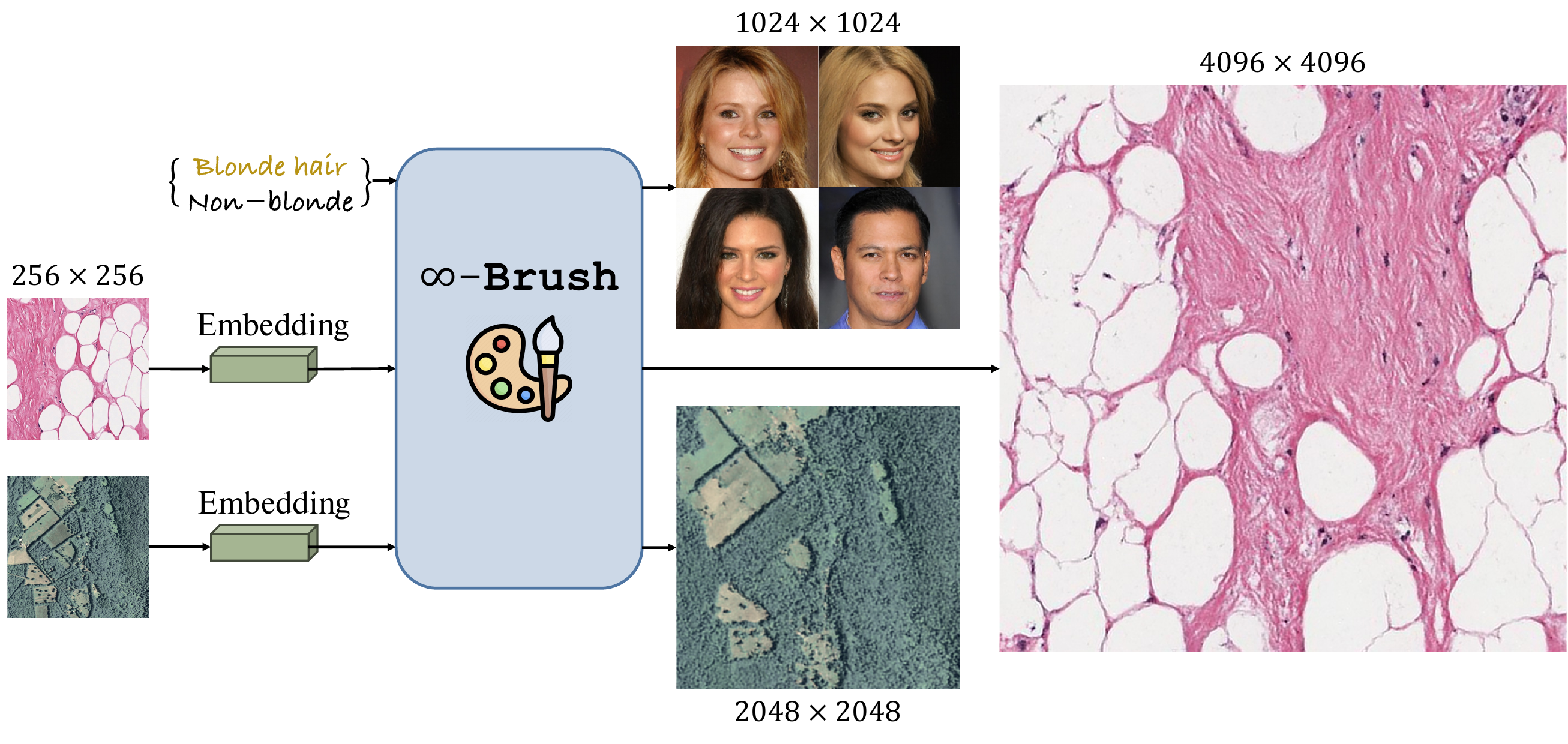}
    \caption{\texttt{$\infty$-Brush} is able to controllably generate images at arbitrary resolutions of up to $4096 \times 4096$, conditioned on any available auxiliary information about the images.}
    \label{fig:teaser}
\end{figure}

\section{Introduction}
\label{sec:intro}
Diffusion models are powerful generative models that have achieved remarkable success in synthesizing diverse and complex data, such as images and audio \cite{nichol2021improved, liu2023audioldm}. Despite their success, it is still difficult to generate high-resolution images, especially when it is necessary to condition them on intricate, domain-specific information. Practical histopathology and satellite imagery applications in medical diagnostics, environmental monitoring, and beyond require precise and controllable very large image synthesis -- well beyond $1024 \times 1024$ pixels, which is impractical with the current state-of-the-art (SoTA) models such as Stable Diffusion-XL (SDXL) \cite{podell2023sdxl}.

The SoTA methods for controllable large image generation still exhibit significant limitations. We can distinguish two broad categories; the first set of approaches directly employs conditional diffusion models in finite latent or pixel space and is inherently limited by its design to generate images only at the resolution on which it was trained. Examples are SDXL \cite{podell2023sdxl} and Matryoshka Diffusion \cite{gu2023matryoshka} which can produce images at a resolution of up to $1024 \times 1024$ pixels. While impressive, such methods cannot generate images at resolutions higher than those they were trained in, without a loss in quality or fidelity. Additionally, as the resolution increases, the computational resources required to train and run these models scale quadratically, making the process increasingly inefficient for larger image sizes. 

The second strategy, introduced by MultiDiffusion \cite{bar-tal2023multidiffusion} and adapted by Graikos \etal \cite{graikos2023learned}, involves a patch-based method that splits large image generation into smaller segments. This technique involves training "local" diffusion models on the patches from large images and performing large image synthesis using an outpainting algorithm. While this approach is computationally more efficient and produces sufficiently realistic larger images, it falls short of capturing long-range spatial dependencies (as discussed in the supplementary). This limitation stems from the heavy reliance on local information, as the generation of each patch is predominantly influenced by the local conditioning and not affected by the information of far away patches.

The previously mentioned methods operate in finite image or latent space and cannot significantly exceed the training image sizes during generation. This makes training the models directly on the entire large images a necessity, leading to insurmountable computational costs. Recently, Bond-Taylor \etal \cite{bond2023infty} have demonstrated that by representing images as functions in Hilbert space $\mathcal{H}$ they can synthesize arbitrarily-large images while training on fixed-size inputs. However, their models can not be conditioned for controllable image generation, which is necessary to efficiently utilize the model in downstream applications, such as data augmentation.





In this work, we propose a novel conditional diffusion model, \texttt{$\infty$-Brush}, for controllable large image synthesis in function space. Our model learns to synthesize images as continuous functions at arbitrarily sampled coordinates, enabling the generation of images at any desired resolution. 

To condition the infinite-dimensional diffusion models we propose a cross-attention neural operator in function space. This operator is necessary since naively trying to condition the diffusion process using existing cross-attention operations on fixed pixel grids is inadequate. Similar to how synthesizing images in finite dimensions cannot capture long-range and intricate details, applying cross-attention on a fixed grid will result in the loss of fine detail. In our experiments, we compare our proposed neural operator to conventional cross-attention and show how we can better capture fine details across all image scales. 

The $\infty$-Diff \cite{bond2023infty} model samples 25\% of the image pixels during training. Directly applying it to large images is infeasible due to memory constraints. Instead, we show that we can train our single, conditional model on much smaller subsets of pixels $(0.4\%)$ from each large image without loss in generation quality. This enables us to apply the infinite-dimensional diffusion model on large image datasets, where images can be up to $4096 \times 4096$ pixels.

In our experiments, we first demonstrate our infinite-dimensional conditioning mechanism, the cross-attention neural operator, by performing conditional image generation on CelebA-HQ \cite{karras2018progressive}. 
We then showcase large image generation, where we train models on histopathology and satellite image datasets and demonstrate how our method \texttt{$\infty$-Brush} outperforms patch-based generation \cite{graikos2023learned} in terms of maintaining global structure without sacrificing local fidelity. In these large image domains, the $4096 \times 4096$ resolution that we achieve is not attainable by any existing model \cite{graikos2023learned, bar-tal2023multidiffusion, podell2023sdxl}.
Figure~\ref{fig:teaser} illustrates that \texttt{$\infty$-Brush} is able to controllably synthesize images at arbitrary resolutions of up to $4096 \times 4096$. 



In summary, our contributions are as follows:
\begin{itemize}
    \item We propose a cross-attention neural operator in function space. This operator allows for the incorporation of external information during image generation. 
    \item We use this operator to build a conditional denoiser in function space as part of \texttt{$\infty$-Brush}, the first conditional diffusion model in function space.
    \item We ensure tractable training of our model on very large images by only training on $0.4\%$ subsets of pixels while inferring at arbitrary resolutions.
    \item We show how our method generates images at the hencetofore infeasible size of $4096\times4096$ pixels while maintaining both global structure and fine details.
\end{itemize}

\section{Related Work}
\textbf{Controllable Generation with Diffusion Models.} Diffusion models \cite{pmlr-v37-sohl-dickstein15, ho2020denoising, Le_Nguyen_Le_Do_Do_Tran_2024} synthesize data by reversing a diffusion process. Latent Diffusion Models (LDMs) \cite{rombach2022high} operate in a lower-dimensional latent space rather than pixel space, significantly reducing the computational load and enabling the generation of high-quality images. Controllable generation is achieved by conditioning on desired attributes, such as class-conditioning \cite{nichol2021improved}, gradient-based guidance \cite{dhariwal2021diffusion}, and classifier-free guidance \cite{ho2022classifier}.

\noindent\textbf{Large Image Generation.} SDXL \cite{podell2023sdxl} makes a step towards large image generation with its ability to generate higher-resolution images. However, controllable generation with SDXL cannot scale to large images because it is constrained to synthesize images only at the resolution on which it was trained ($1024 \times 1024$), leading to a quadratic increase in computational demands with resolution. Our diffusion model, \texttt{$\infty$-Brush}, learns to controllably synthesize images in function space which enables us to generate large images at any desired resolution of up to $4096 \times 4096$ by only training on subsets of $65536$ pixels. 

The patch-based approach for controllable generation, exemplified by MultiDiffusion \cite{bar-tal2023multidiffusion} and adapted in \cite{graikos2023learned}, efficiently generates large images by synthesizing individual patches that are later combined. Despite its computational efficiency and ability to produce realistic images, this method struggles to capture long-range spatial dependencies due to its use of only local information. In contrast, our model operates on the \textbf{entirety} of the image, as represented by a function, maintaining large-scale structures and long-range dependencies.

\noindent\textbf{Diffusion Models in Infinite Dimensions.} Kerrigan \etal \cite{pmlr-v206-kerrigan23a} introduced the concept of applying diffusion models to functional data, pioneering the idea that generative models can operate beyond the confines of finite-dimensional spaces. Building on the ideas of infinite-dimensional diffusion, Lim \etal \cite{lim2023scorebased} and $\infty$-Diff \cite{bond2023infty} specifically address the generation of images represented in function space. However, the methods cannot be conditioned for controllable image generation. To the best of our knowledge, our \texttt{$\infty$-Brush} with a novel cross-attention neural operator is the first conditional diffusion model in infinite dimensions designed for controllable large image synthesis.

\section{Preliminaries}
\subsection{Notation and Data}
\label{subsection:notation}
Let $(\mathcal{X}, \mathcal{A}, \mu)$ be a measure space where $\mathcal{X} \subseteq \mathbb{R}^{d_x}$, $\mathcal{A}$ is a $\sigma$-algebra on the set $\mathcal{X}$ and $\mu$ is a measure on $(\mathcal{X}, \mathcal{A})$. Let $\mathcal{H}$ be a separable Hilbert space over the domain $\mathcal{X}$, equipped with its Borel $\sigma$-algebra $\mathcal{B}(\mathcal{H})$. For simplicity, we consider the case where $\mathcal{H}$ is the space of $L^2$ functions $\mathcal{H} = L^2(\mathcal{X}, \mu)$, which is equipped with its inner product $\langle \mathbf{f}, \mathbf{g}\rangle_{L^2(\mathcal{X}, \mu)} = \int_\mathcal{X}\mathbf{f}\mathbf{g}\mathrm{d}\mu$. It is worth noting that our method is agnostic to the choice of $\mathcal{H}$ and can be applied to other spaces.

Assuming that we have a dataset of the form $\mathcal{D} = \{(\mathbf{u}_k, \mathbf{e}_k)\}_{1 \le k \le D}$, where each $\mathbf{u}_j \in \mathcal{H}$ is an i.i.d. draw from an unknown probability measure $\mathbb{Q}_\mathrm{data}$ on $\mathcal{H}$ and $\mathbf{e}_j$ is a control component of the corresponding function $\mathbf{u}_j$. In our experiment settings, $\mathbf{e}_j$ can be a label or an embedding vector (from vision-language or self-supervised models) with finite dimensions. 

In practice, it is difficult to represent the function directly and instead, for an input function $\mathbf{u}_j$, we discretize it on the mesh $\mathbf{x}_j = \{\mathbf{x}_j^{(i)}\}_{1 \le i \le N} \subset \mathcal{X}$, which is a discrete subset on $\mathcal{X}$ with corresponding discretized observations $\bigl\{\mathbf{u}_j\bigl(\mathbf{x}_j^{(i)}\bigl)\bigl\}_{1 \le i \le N}$, being the output of function $\mathbf{u}_j$ at the $i$-th observation point.  

\subsection{Gaussian Measures on Hilbert Spaces}
Let $\mathbb{Q}$ be a probability measure on $(\mathcal{H}, \mathcal{B}(\mathcal{H}))$. If $\mathbb{Q}$ is Gaussian, then there exists a mean element $\mathbf{m} \in \mathcal{H}$ and a covariance operator $\mathbf{C}: \mathcal{H} \rightarrow \mathcal{H}$, such that
\begin{equation}
    \int_\mathcal{H} \langle \mathbf{u}, \mathbf{x}\rangle \mathbb{Q}(\mathrm{d}\mathbf{x}) = \langle \mathbf{m}, \mathbf{u} \rangle,\quad\forall \mathbf{u} \in \mathcal{H},    
\end{equation}
\begin{equation}
    \int_\mathcal{H}\langle \mathbf{u}_1, \mathbf{x} - \mathbf{m} \rangle \langle \mathbf{u}_2, \mathbf{x} - \mathbf{m}\rangle \mathbb{Q}(\mathrm{d}\mathbf{x}) = \langle \mathbf{C}\mathbf{u}_1,\mathbf{u}_2 \rangle,\quad\forall \mathbf{u}_1, \mathbf{u}_2 \in \mathcal{H}. 
\end{equation}
The covariance operator $\mathbf{C}$ is symmetric, positive semi-definite, compact, and has finite trace $\mathrm{Tr}(\mathbf{C}) < +\infty$. Conversely, let $C$ be a positive, symmetric, trace class operator in $\mathcal{H}$ and let $\mathbf{m} \in \mathcal{H}$, then there exists a Gaussian measure in $\mathcal{H}$ with mean $\mathbf{m}$ and covariance $\mathbf{C}$ \cite{DaPrato_Zabczyk_2014}. From now on, we will denote $\mathbb{Q} = \mathcal{N}(\mathbf{m}, \mathbf{C})$ for such a Gaussian measure. 
\subsection{Diffusion Models in Function Space}
Here we briefly describe diffusion probabilistic models in function space $\mathcal{H}$ \cite{pmlr-v206-kerrigan23a, lim2023scorebased, bond2023infty}, which is constructed similarly to that of DDPMs \cite{ho2020denoising}. Note that the key difference is that diffusion models in function space operate in infinite dimensions.

\noindent\textbf{Forward process.} The forward process of a diffusion model in function space is defined as a discrete-time Markov chain that incrementally perturbs probability measure $\mathbb{Q}_{\mathrm{data}}$ towards a Gaussian measure $\mathcal{N}(\mathbf{m}, \mathbf{C})$ with a zero mean and a specified covariance operator $\mathbf{C}$. It is a time-indexed process where each step $\mathbf{u}_t$ is obtained by applying a transformation to the previous step $\mathbf{u}_{t-1}$, which involves a scaling factor $\sqrt{1-\beta_t}\mathbf{u}_{t-1}$ related to the variance schedule $\beta$, and adding scaled Gaussian noise $\sqrt{\beta_t}\bm{\xi}_t$ with $\bm{\xi}_t \sim \mathcal{N}(\mathbf{0}, \mathbf{C})$:
\begin{equation}
    \mathbf{u}_t = \sqrt{1-\beta_t}\mathbf{u}_{t-1} + \sqrt{\beta_t}\bm{\xi}_t \quad\quad t=1,2,\dots,T.
\end{equation}
Similar to diffusion models in finite dimensions, the forward process in function space also admits sampling $\mathbf{u}_t$ at an arbitrary timestep $t$ in closed form. For $\Bar{\alpha}_t = \prod_{i=1}^t(1-\beta_t)$, we have:
\begin{equation}
    \mathbb{Q}\left(\mathbf{u}_t\vert\mathbf{u}_0\right) = \mathcal{N}\left(\mathbf{u}_t; \sqrt{\Bar{\alpha}_t}\mathbf{u}_0, (1-\Bar{\alpha}_t)\mathbf{C}\right).
\end{equation}

\noindent\textbf{Reverse process.} The reverse process in the diffusion model iteratively denoises from the Gaussian measure $\mathcal{N}(\mathbf{m},\mathbf{C})$ back to the probability measure $\mathbb{Q}_0 = \mathbb{Q}_{\mathrm{data}}$. This is achieved by sampling from the reverse-time transition measures $\mathbb{Q}\left(\mathbf{u}_{t-1}\vert\mathbf{u}_t\right)$, approximated by a Gaussian measure with parameters $\theta$ due to the intractable normalization constant of the Bayes' rule: 
\begin{equation}
    \mathbb{P}_\theta\left(\mathbf{u}_{t-1}\vert\mathbf{u}_t\right) = \mathcal{N}\left(\mathbf{u}_{t-1}; \mathbf{m}_\theta(\mathbf{u}_t, t), \mathbf{C}_\theta(\mathbf{u}_t, t)\right)
\end{equation}
Likewise, we are able to derive a closed-form representation of the forward process posteriors, which are tractable when conditioned on $\mathbf{u}_0$:
\begin{equation}
    \mathbb{Q}\left(\mathbf{u}_{t-1}\vert\mathbf{u}_t, \mathbf{u}_0\right) = \mathcal{N}\left(\mathbf{u}_{t-1};\mathbf{\tilde{m}}_t(\mathbf{u}_t, \mathbf{u}_0), \Tilde{\beta}_t\mathbf{C}\right),
\end{equation}
where $\mathbf{\tilde{m}}_t(\mathbf{u}_t,\mathbf{u}_0) = \frac{\sqrt{\Bar{\alpha}_{t-1}}\beta_t}{1-\Bar{\alpha}_t}\mathbf{u}_0 + \frac{\sqrt{1-\beta_t}(1-\Bar{\alpha}_{t-1})}{1-\Bar{\alpha}_t}\mathbf{u}_t$ and $\Tilde{\beta}_t = \frac{1-\Bar{\alpha}_{t-1}}{1-\Bar{\alpha}_t}\beta_t$.

\noindent\textbf{Training objective.} Similar to DDPM in finite dimensions, the reparameterization is used to achieve better training, which results in a simplified loss:
\begin{equation}
    \mathcal{L}_{\mathrm{simple}} = \left\vert\left\vert \mathbf{C}^{-1/2}\left(\bm{\xi}_t - \bm{\xi}_\theta(\mathbf{u}_t, t)\right)\right\vert\right\vert_\mathcal{H}^2.
\end{equation}

\subsection{Neural Operators} 
Neural operators \cite{kovachi2024neuraloperator, cao2021choose, hao2023gnot} are a type of neural network tailored to learn mappings between infinite-dimensional function spaces. In the context of diffusion models in infinite dimensions, a denoiser is parameterized by a neural operator $\mathcal{G}_\theta: \mathcal{U}^* \rightarrow \mathcal{U}$, which learns to map from noisy function space $\mathcal{U}^*$ to denoised function space $\mathcal{U}$. With $\mathbf{u} \in \mathcal{U}^*$ and $\mathbf{s} \in \mathcal{U}$, we access their pointwise evaluations by discretizing them on the mesh $\mathbf{x} = \{\mathbf{x}^{(i)}\}_{1 \le i \le N} \subset \mathcal{X}$. Neural operators include multiple operator layers akin to those in a finite-dimensional neural network $\mathbf{v}_0 \mapsto \mathbf{v}_1 \mapsto \cdots \mapsto \mathbf{v}_L$, where layer $\mathbf{v}_l \mapsto \mathbf{v}_{l+1}$ is built upon a local linear operator, a non-local integral kernel operator, and a bias function:
\begin{equation}
    \mathbf{v}_{l+1}(\mathbf{x}^{(i)}) = \sigma_{l+1}\left(W_l\mathbf{v}_l(\mathbf{x}^{(i)}) + (\mathcal{K}_l(\mathbf{u}; \phi)\mathbf{v}_l)(\mathbf{x}^{(i)}) + b_l(\mathbf{x}^{(i)})\right),
\end{equation}
with $\mathcal{K}_l(\mathbf{u}; \phi)$ being an integral kernel operator aggregating information spatially.
\section{The Proposed Method} 
We propose a novel conditional diffusion model in function space $\mathcal{H}$. Based on the background provided, we now formulate the forward and reverse process and the training objective of our conditional diffusion model in infinite dimensions. Furthermore, we present a novel architecture to parameterize the denoising process with a conditional denoiser equipped with cross-attention neural operators. 

\subsection{Conditional Diffusion Models in Function Space}
In the context of image generation, we discretize the function $\mathbf{u}_j$ on the mesh $\mathbf{x}_j = \{\mathbf{x}_j^{(i)}\}_{1 \le i \le N} \subset \mathcal{X}$ by sampling $N$ coordinates of each image, which results in non-smooth input space. To achieve a smoother function representation, a smoothing operator \cite{hoogeboom2023blurring, rissanen2023generative} $\mathbf{A}: \mathcal{H} \rightarrow \mathcal{H}$, \eg a truncated Gaussian kernel, is applied to approximate the rough inputs within the function space $\mathcal{H}$.

\noindent\textbf{Forward process.} The forward process of our conditional diffusion model in infinite dimensions is equivalent to that of an unconditional diffusion model in function space, which gradually perturbs the probability measure $\mathbb{Q}_0 = \mathbb{Q}_\mathrm{data}$ towards a Gaussian measure $\mathcal{N}(\mathbf{m}, \mathbf{C})$ and enables sampling at any arbitrary timestep $t$ with $\Bar{\alpha}_t = \prod_{i=1}^t(1-\beta_t)$:
\begin{equation}
    \mathbb{Q}\left(\mathbf{u}_t\vert\mathbf{u}_0\right) = \mathcal{N}\left(\mathbf{u}_t; \sqrt{\Bar{\alpha}_t}\mathbf{A}\mathbf{u}_0, (1-\Bar{\alpha}_t)\mathbf{A}\mathbf{C}\mathbf{A}^T\right). 
\end{equation}

\noindent\textbf{Reverse process.} We use a variational approach to approximate posterior measures with a variational family of measures on $\mathcal{H}$ and incorporate the conditional embedding $\mathbf{e}$ to control the generation process. We model the underlying posterior measure $\mathbb{Q}(\mathbf{u}_{t-1}\vert\mathbf{u}_t)$ with a conditional Gaussian measure:
\begin{equation}
    \mathbb{P}_\theta(\mathbf{u}_{t-1}\vert\mathbf{u}_t, \mathbf{e}) = \mathcal{N}\left(\mathbf{u}_{t-1}; \mathbf{m}_\theta(\mathbf{u}_t, \mathbf{e}, t), \mathbf{A}\mathbf{C}_\theta(\mathbf{u}_t, \mathbf{e}, t)\mathbf{A}^T\right).
\end{equation}

\begin{proposition}[Learning Objective]{The cross-entropy of conditional diffusion models in function space has a variational upper bound of}
\begin{align}
    \mathcal{L}_{\mathrm{CE}} = -\mathbb{E}_\mathbb{Q}\log{\mathbb{P}_\theta(\mathbf{u}_0\vert\mathbf{e})} &\le \mathbb{E}_\mathbb{Q} \Bigg[\underbrace{\mathrm{KL}(\mathbb{Q}(\mathbf{u}_T \vert \mathbf{u}_0) \parallel \mathbb{P}_\theta(\mathbf{u}_T))}_{\mathcal{L}_T} \underbrace{-\log \mathbb{P}_\theta(\mathbf{u}_0 \vert \mathbf{u}_1, \mathbf{e})}_{\mathcal{L}_0} \notag\\
&+ \sum_{t=2}^{T} \underbrace{\mathrm{KL}(\mathbb{Q}(\mathbf{u}_{t-1} \vert \mathbf{u}_{t}, \mathbf{u}_0) \parallel \mathbb{P}_\theta(\mathbf{u}_{t-1} \vert\mathbf{u}_{t}, \mathbf{e})}_{\mathcal{L}_{t-1}}\Bigg].
\end{align}
\label{proposition_main:objective}
\end{proposition}
\begin{proof}
Please refer to the Appendix~\ref{appendix:proof} for the full proof.\qed    
\end{proof}

To compute the KL divergence between probability measures $\mathrm{KL}(\mathbb{Q} \parallel \mathbb{P})$, we need to utilize a measure-theoretic definition of the KL divergence, which is stated in the following lemmas \cite{DaPrato_Zabczyk_2014}. 

\begin{lemma}[Measure Equivalence - The Feldman-Hájek Theorem]
    Let $\mathbb{Q} = \mathcal{N}(\mathbf{m}_1,\mathbf{C}_1)$ and $\mathbb{P} = \mathcal{N}(\mathbf{m}_2,\mathbf{C}_2)$ be Gaussian measures on $\mathcal{H}$. They are equivalent if and only if $(i): \mathbf{C}^{1/2}_1(\mathcal{H}) = \mathbf{C}^{1/2}_2(\mathcal{H}) = \mathcal{H}_0$, $(ii): \mathbf{m}_1 - \mathbf{m}_2 \in \mathcal{H}_0$, and $(iii):$ The operator $(\mathbf{C}^{-1/2}_1\mathbf{C}^{1/2}_2)(\mathbf{C}^{-1/2}_1\mathbf{C}^{1/2}_2)^* -\mathbf{I}$ is a Hilbert-Schmidt operator on the closure $\overline{\mathcal{H}_0}$.
\label{lemma_main:measure_equi}
\end{lemma}

\begin{lemma}[The Radon-Nikodym Derivative]
Let $\mathbb{Q} = \mathcal{N}(\mathbf{m}_1,\mathbf{C}_1)$ and $\mathbb{P} = \mathcal{N}(\mathbf{m}_2,\mathbf{C}_2)$ be Gaussian measures on $\mathcal{H}$. If $\mathbb{P}$ and $\mathbb{Q}$ are equivalent and $\mathbf{C}_1 = \mathbf{C}_2 = \mathbf{C}$, then $\mathbb{P}$-a.s. the Radon-Nikodym derivative $\mathrm{d}\mathbb{Q}/\mathrm{d}\mathbb{P}$ is given by
\begin{equation}    
\resizebox{\textwidth}{!}{
$\frac{\mathrm{d} \mathbb{Q}}{\mathrm{d} \mathbb{P}}(\mathbf{f})=\exp \left[\left\langle \mathbf{C}^{-1 / 2}\left(\mathbf{m}_1-\mathbf{m}_2\right), \mathbf{C}^{-1 / 2}\left(\mathbf{f}-\mathbf{m}_2\right)\right\rangle-\frac{1}{2} \| \mathbf{C}^{-1 / 2} (\mathbf{m}_1-\mathbf{m}_2) \|^2\right] \forall \mathbf{f} \in \mathcal{H}.$
} 
\end{equation}
\label{lemma_main:radon}
\end{lemma}
\begin{proof}
    The proof of both lemmas is in the Appendix~\ref{appendix:proof}.\qed
\end{proof}

Lemma~\ref{lemma_main:measure_equi} states the three conditions for the equivalence of two Gaussian measures. Lemma~\ref{lemma_main:radon}, a consequence of the Feldman-Hájek theorem, provides the Radon-Nikodym derivative formula for Gaussian measures on $\mathcal{H}$.

To train the diffusion model in functional space we have to minimize the upper bound of Proposition~\ref{proposition_main:objective}, which requires us to compute the KL divergence between the measures $\mathbb{Q}, \mathbb{P}$. In order to satisfy Lemma~\ref{lemma_main:measure_equi}, which will enable us to use Lemma~\ref{lemma_main:radon} to compute the KL divergence, we make the following assumption:
\begin{assumption}
Let $\mathbb{Q} = \mathcal{N}(\mathbf{\tilde{m}}_t(\mathbf{u}_t, \mathbf{u}_0), \Tilde{\beta}_t\mathbf{C})$ and $\mathbb{P}_\theta = \mathcal{N}(\mathbf{m}_\theta(\mathbf{u}_t, \mathbf{e}, t), \Tilde{\beta}_t\mathbf{C})$ be Gaussian measures on $\mathcal{H}$. With a conditional component $\mathbf{e}$, which can be an element of finite-dimensional space $\mathbb{R}^d$ or Hilbert space $\mathcal{H}$, there exists a parameter set $\theta$ such that the difference in mean elements of the two measures falls within the scaled covariance space: 
\label{assumption_main:measure_equi}
\begin{equation}
\mathbf{\tilde{m}}_t(\mathbf{u}_t, \mathbf{u}_0)-\mathbf{m}_\theta(\mathbf{u}_t, \mathbf{e}, t) \in (\Tilde{\beta}_t\mathbf{C})^{1/2}(\mathcal{H}).    
\end{equation}
\end{assumption}

By making this assumption we satisfy all three conditions of Lemma~\ref{lemma_main:measure_equi}:
$(i): \mathbf{C}^{1/2}_1(\mathcal{H}) = \mathbf{C}^{1/2}_2(\mathcal{H}) = (\Tilde{\beta}_t\mathbf{C})^{1/2}(\mathcal{H}) = \mathcal{H}_0$; 
$(ii):  \mathbf{m}_1 - \mathbf{m}_2 \in \mathcal{H}_0$  is directly satisfied from Assumption~\ref{assumption_main:measure_equi}; 
$(iii): (\mathbf{C}^{-1/2}_1\mathbf{C}^{1/2}_2)(\mathbf{C}^{-1/2}_1\mathbf{C}^{1/2}_2)^* -\mathbf{I} = \mathbf{I} - \mathbf{I}$ is the zero operator, which is trivially a Hilbert-Schmidt operator as its Hilbert-Schmidt norm is $0$.
As a consequence, $\mathbb{Q}$ and $\mathbb{P}$ are equivalent, allowing us to utilize the Radon-Nikodym derivative from Lemma~\ref{lemma_main:radon}.

\begin{theorem}[Conditional Diffusion Optimality in Function Space]\\
    Given the specified conditions in Assumption~\ref{assumption_main:measure_equi}, the minimization of the learning objective in Proposition~\ref{proposition_main:objective} is equivalent to obtaining the parameter set $\theta^*$ that is the solution to the problem
    {\small
    \begin{equation}
       \theta^* = \argmin_\theta \mathbb{E}_{\mathbf{u}_0\sim\mathbb{Q}_{\mathrm{data}}, t\sim[1, T]} \lambda_t\left\vert\left\vert \mathbf{C}^{-1/2}\left(\mathbf{A}\bm{\xi} - \bm{\xi}_\theta(\sqrt{\Bar{\alpha}_t}\mathbf{A}\mathbf{u}_0 + \sqrt{1-\Bar{\alpha}_t}\mathbf{A}\bm{\xi}, \mathbf{e}
       ,t)\right)\right\vert\right\vert_\mathcal{H}^2,
    \end{equation}}
where $\bm{\xi}\sim\mathcal{N}(\mathbf{0}, \mathbf{C})$, $\mathbf{A}: \mathcal{H}\rightarrow\mathcal{H}$ denotes a smoothing operator, $\mathbf{e} \in (\mathbb{R}^d \cup \mathcal{H})$ is a conditional component, $\bm{\xi}_\theta : \{1, 2, \dots, T\} \times (\mathbb{R}^d \cup \mathcal{H}) \times \mathcal{H} \rightarrow \mathcal{H}$ is a parameterized mapping, $\lambda_t = \beta^2_t/2\Tilde{\beta}_t(1-\beta_t)(1-\Bar{\alpha}_t) \in \mathbb{R}$ is a time-dependent constant.
\end{theorem}
\begin{proof}
    Please refer to the Appendix~\ref{appendix:proof} for the full proof. \qed
\end{proof}
\subsection{Conditional Denoiser with Cross-Attention Neural Operators}
\begin{figure}[!t]
    \centering
    \includegraphics[width=\linewidth]{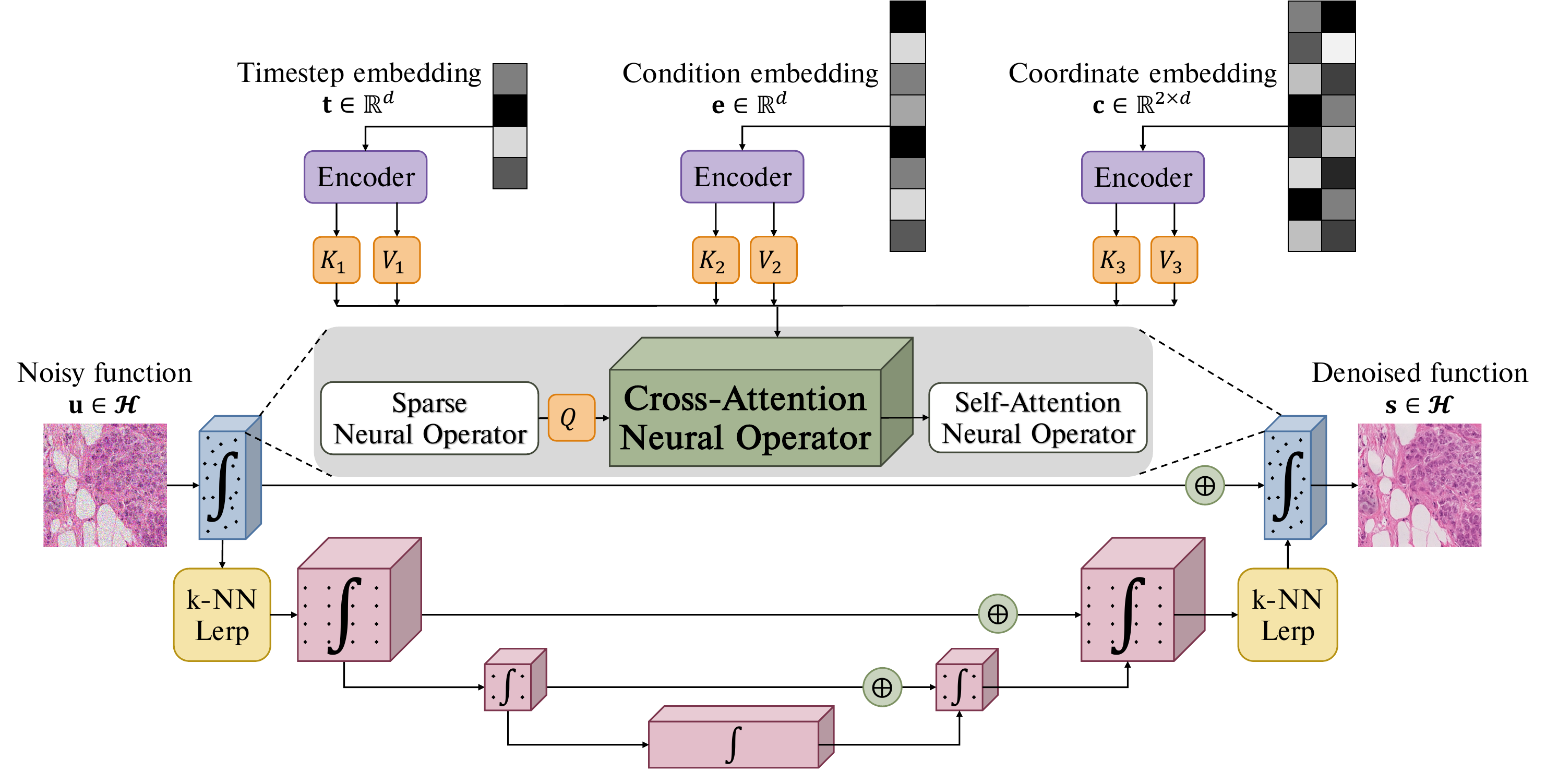}
    \caption{Given a noisy function $\mathbf{u} \in \mathcal{H}$, we discretize it by randomly selecting a subset of coordinates $\mathbf{x} = \{\mathbf{x}^{(i)}\}_{1 \le i \le N} \subset \mathcal{X}$ then feed it into our conditional denoiser returning a denoised function $\mathbf{s} \in \mathcal{H}$. The conditional denoiser architecture of \texttt{$\infty$-Brush} includes a sparse level and a grid level. The sparse level (in blue) utilizes a sparse neural operator, a cross-attention neural operator, and a self-attention neural operator, focusing on capturing fine-grained details. The grid level (in pink) targets global information. We use k-NN linear interpolation to transform the sparse points to a regularly spaced grid.}
    \label{fig:denoiser}
\end{figure}

Our \texttt{$\infty$-Brush} utilizes a hierarchical denoiser architecture including a \textit{sparse level} for efficiently capturing fine-grained details and a \textit{grid level} for global information (Fig.~\ref{fig:denoiser}). We discretize the noisy functions $\mathbf{u} \in \mathcal{H}$ and denoised functions $\mathbf{s} \in \mathcal{H}$ by randomly selecting a subset of coordinates $\mathbf{x} = \{\mathbf{x}^{(i)}\}_{1 \le i \le N} \subset \mathcal{X}$. At the sparse level, we successively apply a sparse neural operator, our cross-attention neural operator, and self-attention on pointwise evaluations of the function. 

The computational complexity of the vanilla attention is quadratic $\mathcal{O}(N^2d)$ w.r.t. the sequence length, or number of function samples (in this case $N$), and linear w.r.t. their dimension $d$. For learning operators in infinite dimensions, $N$ could go up to millions of points (\eg when generating $4096 \times 4096$ images, $N \approx 16$ million points). We address that problem by proposing a cross-attention neural operator of linear complexity with respect to $N$.

Specifically in the cross-attention neural operator, suppose we have $L$ conditional embeddings $\{Y_l \in \mathbb{R}^{N_l \times d}\}_{1 \le l \le L}$. In our \texttt{$\infty$-Brush}, $L=3$ representing the diffusion timestep embedding $\mathbf{t}$, condition embedding $\mathbf{e}$, and coordinate embedding $\mathbf{c}$. First, we compute the queries $Q = (\mathbf{q}_i)$, keys $K_l = (\mathbf{k}_i^l) = Y_lW_k$, and values $V_l = (\mathbf{v}_i^l) = Y_lW_v$, then normalize all $\mathbf{q}_i$ and $\mathbf{k}_i$ to be $\Tilde{\mathbf{q}}_i = \texttt{softmax}(\mathbf{q}_i)$ and $\Tilde{\mathbf{k}}_i = \texttt{softmax}(\mathbf{k}_i)$. Finally, cross-attention is
\begin{equation}
    \mathbf{z}_l = \tilde{\mathbf{q}}_t+\frac{1}{L} \sum_{l=1}^L \sum_{i=1}^{N_l} \alpha_t^l\left(\tilde{\mathbf{q}}_t \cdot \tilde{\mathbf{k}}_i^l\right) \mathbf{v}_i^l =  \tilde{\mathbf{q}}_t+\frac{1}{L} \sum_{l=1}^L \alpha_t^l \tilde{\mathbf{q}}_t \cdot\left(\sum_{i=1}^{N_l} \tilde{\mathbf{k}}_i^l \odot \mathbf{v}_i^l\right),
    \label{eq:CANO}
\end{equation}
where $\alpha^l_t = 1/\sum_{j = 1}^{N_l} \Tilde{\mathbf{q}}_t \cdot \Tilde{\mathbf{k}}_j$ is the normalization coefficient. The key difference with vanilla attention is that we first multiply pointwise vectors $\tilde{\mathbf{k}}_i^l$ and $ \mathbf{v}_i^l$, and compute the dot product with $\tilde{\mathbf{q}}_t $ afterward. Hence, the complexity of Eq.~\ref{eq:CANO} is $\mathcal{O}((N + \sum_l N_l)d^2)$, which is linear w.r.t. the number of points $N$.

The output of the sparse level is linearly interpolated to a regularly spaced grid using k-Nearest Neighbors, which is the input to the grid-level model. The grid data points are passed to a grid-based, finite-dimensional UNO architecture ~\cite{lim2023scorebased, bond2023infty} that is utilized to aggregate global information. The UNO architecture is based on the widespread UNet model, which has been widely studied to condition finite-dimensional diffusion models \cite{rombach2022high}. Following that literature, we use the vanilla cross-attention at the bottleneck of the UNet denoiser to integrate the conditional information at the grid level. In the experiments, we show that since the coarse interpolation at the grid level is not a complete representation of the function, the conditioning needs to be applied to both finite-dimensional (grid) and infinite-dimensional (sparse) levels to attain high-quality results.


\section{Experiments}
\subsection{Experimental Settings}
\textbf{Datasets.} We utilize the CelebA-HQ dataset \cite{karras2018progressive} as a testbed for our cross-attention neural operator. We use $30,000$ images at $1024 \times 1024$ resolution along with the facial attribute \textit{blonde/non-blonde hair} to train our conditional diffusion model and compare with the unconditional version \cite{bond2023infty}.

For large image datasets, following the evaluation of \cite{graikos2023learned}, we utilize digital histopathology images from The Genome Cancer Atlas (TCGA) \cite{cancer2013cancer} and satellite imagery from the National Agriculture Imagery Program (NAIP) \cite{naip}. We use the BRCA subset of TCGA which contains breast cancer histopathology images. We select $20\times$ patches of sizes $4096 \times 4096$ and $1024\times 1024$, equivalent in scope to patches from $1.25\times$ and $5\times$ magnifications, respectively. To provide conditioning, we resize these images to $256 \times 256$ and extract embeddings from Quilt \cite{ikezogwo2023quilt}.

We utilize the NAIP images from the Chesapeake Land Cover dataset \cite{robinson2019large}, by extracting $1024 \times 1024$ non-overlapping patches, resulting in $35,000$ satellite images. We train a Vision Transformer (ViT-B/16) \cite{dosovitskiy2020image} on the resized $224 \times 224$ pixel versions of these images using the self-supervised DINO algorithm \cite{caron2021emerging}. We extract the learned DINO embeddings to train our conditional diffusion model on pairs of $1024 \times 1024$ images and corresponding SSL embeddings.

\noindent\textbf{Evaluation metrics.} Following the standard evaluation metrics of MultiDiffusion \cite{bar-tal2023multidiffusion} and \cite{graikos2023learned} for large image synthesis, we evaluate our method's image quality on both global structure and fine detail via FID scores \cite{heusel2017gans} using the Clean-FID implementation \cite{parmar2021cleanfid}. For global structure, we calculate CLIP FID \cite{radford2021learning} between the resized version of generated large images and real images. For fine detail, we randomly take $256 \times 256$ crops from both synthesized and real large images and measure FID (Crop FID) between the two sets of patches.

\subsection{Implementation Details}
We train our \texttt{$\infty$-Brush} from scratch for all experiments. At each training iteration, we randomly select a subset of $256^*256 = 65536$ pixels from the image. Regarding the denoiser architecture, we leverage the implementation of the Sparse Neural Operator \cite{kovachi2024neuraloperator}, the unconditional UNO \cite{bond2023infty}, and the general neural operator \cite{hao2023gnot}. For faster runtime and memory efficiency, we implement our cross-attention neural operator using FlashAttention-2 \cite{dao2024flashattention}. The model is trained using the Adam optimizer with a learning rate of $5e-5$ and $\beta_1 = 0.9$, $\beta_2 = 0.99$, along with an exponential moving average (EMA) rate of $0.995$. During inference, we apply DDIM \cite{song2020denoising} with $50$ steps for all experiments. All \texttt{$\infty$-Brush} models were trained on $4$ NVIDIA A100 GPUs, with a batch size of $20$ per GPU.

\subsection{Experimental Results}
\begin{figure}[!t]
    \centering
    \includegraphics[width=\linewidth]{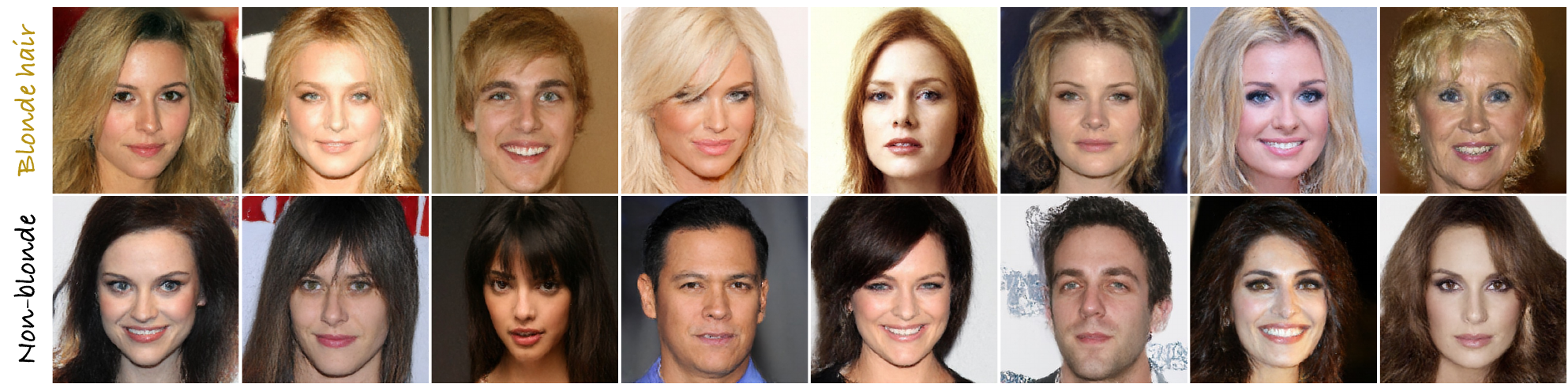}
    \caption{Large images ($1024 \times 1024$) generated from our \texttt{$\infty$-Brush}, conditioned on the facial attribute \textit{blonde/non-blonde} hair.}
    \label{fig:celebahq}
\end{figure}

\begin{table}[!t]
\caption{The CLIP FID scores of our \texttt{$\infty$-Brush} model against $\infty$-Diff showcases our model's capability in conditionally generating celebrity faces on the CelebA-HQ dataset based on the facial attribute of hair color (blonde vs. non-blonde).}
\centering
\footnotesize
\resizebox{1\linewidth}{!}{
\begin{tabular}{ccc|c|c}
\toprule
 \textbf{Dataset} & \textbf{$\#$ Images} & \textbf{Method} & \textbf{Training Config.} & \textbf{CLIP FID} \\ 
 \midrule
 \multirow{2}{*}{\makecell{CelebA-HQ \\ ($1024 \times 1024$)}} & \multirow{2}{*}{$30$k} & $\infty$-Diff \cite{bond2023infty} &  Unconditional & 9.44 \\
\cmidrule{3-5}
 &  & \texttt{$\infty$-Brush} & blonde vs. non-blonde hair & \textbf{8.38} \\
\bottomrule
\end{tabular}
}
\label{table:conditional}
\end{table}

\textbf{Facial Attribute Conditional Generation.} 
We first validate our cross-attention neural operator as an efficient conditioning mechanism for infinite-dimensional diffusion models by adding control to the generation of CelebA-HQ images.

We synthesize $3,000$ images, maintaining the same ratio of blonde/non-blonde as in the entire dataset, and calculate the CLIP FID to assess quality. We compare between our conditional \texttt{$\infty$-Brush} and the unconditional $\infty$-Diff \cite{bond2023infty}. As shown in Table~\ref{table:conditional}, our method outperforms the unconditional model, while also allowing us to control the attribute used as conditioning. Figure~\ref{fig:celebahq} shows examples of large images ($1024 \times 1024$) generated from \texttt{$\infty$-Brush}, conditioned on the \textit{blonde/non-blonde} attribute.

\noindent\textbf{Controllable (Very) Large Image Generation.} We provide experimental results of controllable generation of large ($1024 \times 1024$) and very large ($4096 \times 4096$) images and compare to conditional diffusion models in finite dimensions \cite{podell2023sdxl} and a patch-based approach \cite{graikos2023learned}. In addition, we perform an ablation study to evaluate the significance of our cross-attention neural operator. We further compare the computing resources required for the three different model categories.

Our very large image experiments on the TCGA-BRCA dataset, which has $57$k image patches at a resolution of $4096 \times 4096$ pixels, reveal that \texttt{$\infty$-Brush} excels at capturing the global structure of images, as indicated by the better CLIP FID score (Table~\ref{table:very_large_image}). The model's performance in finer details, reflected in the Crop FID, is slightly worse than the patch-based approach. \texttt{$\infty$-Brush} is trained on batches of just $65536 \approx 0.4\%$ of the pixels from the full images, offering a substantial reduction in complexity, while maintaining both global structure and finer details; see Fig.~\ref{fig_main:large_image} for qualitative results of generated $4096\times 4096$ images.

\begin{figure}[!t]
    \centering
    \includegraphics[width=\linewidth]{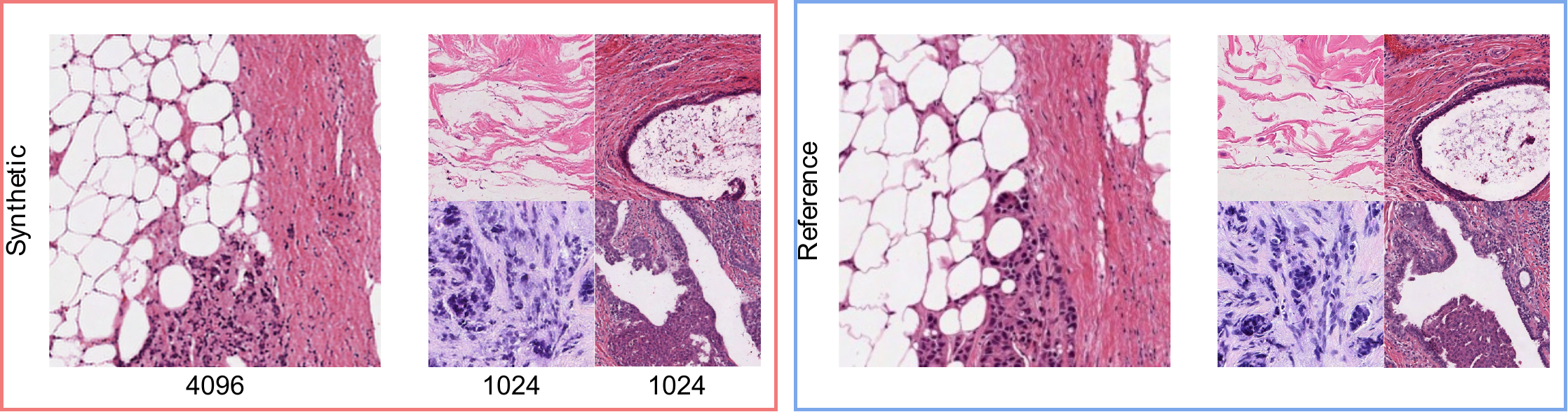}
    \caption{Very large ($4096 \times 4096$) and large ($1024 \times 1024$) images generated from \texttt{$\infty$-Brush}, and the corresponding reference real images used to generate them. Given a single embedding vector of a downsampled $256\times256$ real image, \texttt{$\infty$-Brush} can synthesize images of up to $4096 \times 4096$ and preserve global structures of the reference image.}
    \label{fig_main:large_image}
\end{figure}

\begin{table}[!t]
\centering
\caption{Performance on controllable very large image synthesis on TCGA-BRCA dataset at $4096 \times 4096$ resolution. \texttt{$\infty$-Brush} outperforms the patch-based approach \cite{graikos2023learned} in terms of global structure (CLIP FID) while achieving acceptable local details (Crop FID). SDXL \cite{podell2023sdxl} cannot be trained directly on images of $4096 \times 4096$ images. Additionally, an ablation study on the cross-attention neural operator shows improvement in FID metrics when the proposed mechanism is used. This emphasizes its critical role in the model's ability to synthesize high-resolution images effectively.}
\footnotesize
\resizebox{1\linewidth}{!}{
\begin{tabular}{ccc|c|c|c}
\toprule
 \textbf{Dataset} & \textbf{$\#$ Images} & \textbf{Method} & \textbf{Training Config.} & \textbf{CLIP FID} & \textbf{Crop FID} \\ 
 \midrule
 \multirow{4}{*}{\makecell{BRCA $1.25\times$ \\ ($4096 \times 4096$)}} & \multirow{4}{*}{$57$k} & Graikos \etal~\cite{graikos2023learned} & \makecell{$976$k patches of \\ $1024 \times 1024$}& 2.75 & \textbf{11.30} \\
\cmidrule{3-6}
 &  &  \texttt{$\infty$-Brush} & \multirow{3}{*}{\makecell{$256 ^* 256$ pixels of \\ $57$k full-size images}} & \textbf{2.63} & 14.76 \\
 \cmidrule{3-3} \cmidrule{5-6}
 &  &  \texttt{$\infty$-Brush} &  & \multirow{2}{*}{3.81} & \multirow{2}{*}{16.28} \\
 &  &  \ding{55} Cross-attention neural operator &  & & \\
\bottomrule
\end{tabular}
}
\label{table:very_large_image}
\end{table}

We can also partially attribute our better CLIP FID and worse Crop FID to the different conditioning provided to our model. The $4096 \times 4096$ images are first downsampled to $256\times256$ and a single embedding vector is extracted to capture information from the entire image. In comparison, the patch-based approach of \cite{graikos2023learned} employs $16$ local conditions that describe each of the $16$ patches that form the image, helping the model to focus more on the local appearance.

SDXL \cite{podell2023sdxl} would need to be trained on $4096 \times 4096$ images which is infeasible with most current hardware setups. Thus, we resort to smaller patches from TCGA-BRCA and NAIP, at $1024\times1024$ resolution, to compare with SDXL and the patch-based method of \cite{graikos2023learned}. Table~\ref{table:large_image} shows that our model attains strong global structure fidelity with superior CLIP FID scores, particularly on the BRCA dataset ($3.74$ vs. $6.64$). Despite this, our approach results in higher Crop FID scores, suggesting a trade-off in capturing fine details. While SDXL is able to train at this resolution, it is important to note that \texttt{$\infty$-Brush} achieves similar performance with only a subset of the pixels, because of its highly efficient training process. Fig.~\ref{fig_main:large_image} showcases qualitative results of generated $1024\times 1024$ images from our model. Note that SDXL was pre-trained on LAION-5B, a $5$ billion image caption pair dataset \cite{schuhmann2022laion}, whereas our model was trained from scratch.


\begin{table}[!t]
\centering
\footnotesize
\caption{Performance on controllable large image synthesis on BRCA $5\times$ and NAIP dataset at $1024 \times 1024$ resolution. \texttt{$\infty$-Brush} outperforms other methods in global structure accuracy, with a marginal trade-off in fine detail as reflected in Crop FID.}
\resizebox{1\linewidth}{!}{
\begin{tabular}{ccc|c|c|c}
\toprule
 \textbf{Dataset} & \textbf{$\#$ Images} & \textbf{Method} & \textbf{Training Config.} & \textbf{CLIP FID} & \textbf{Crop FID} \\ 
 \midrule
 \multirow{5}{*}{\makecell{BRCA $5\times$ \\ ($1024 \times 1024$)}} & \multirow{5}{*}{$976$k} & SDXL \cite{podell2023sdxl} &  $976$k full-size images & 6.64 & \textbf{6.98} \\
 \cmidrule{3-6}
  &  & Graikos \etal~\cite{graikos2023learned} & \makecell{$15$M patches of \\ $256 \times 256$}& 7.43 & 15.51 \\
\cmidrule{3-6}
 &  &  \texttt{$\infty$-Brush} & \makecell{$256 ^* 256$ pixels of \\ $976$k full-size images} & \textbf{3.74} & 17.87 \\
\midrule
 \multirow{5}{*}{\makecell{NAIP \\ ($1024 \times 1024$)}} & \multirow{5}{*}{$35$k} & SDXL \cite{podell2023sdxl} &  $35$k full-size images & 10.90 & \textbf{11.50} \\
 \cmidrule{3-6}
  &  & Graikos \etal~\cite{graikos2023learned} & \makecell{$667$k patches of \\ $256 \times 256$}& 6.86 & 43.76 \\
\cmidrule{3-6}
 &  &  \texttt{$\infty$-Brush} & \makecell{$256 ^* 256$ pixels of \\ $35$k full-size images} & \textbf{6.32} & 48.65 \\
 \bottomrule
\end{tabular}
}
\label{table:large_image}
\end{table}

\noindent\textbf{Effectiveness of the Cross-Attention Neural Operator.} We evaluate the cross-attention neural operator's advantage in the \texttt{$\infty$-Brush} model by comparing its performance on the TCGA-BRCA $4096 \times 4096$ dataset with and without this operator. When the neural operator is removed, we use vanilla cross-attention between the conditioning vector and the UNet's bottleneck layer. In Table~\ref{table:very_large_image}, we observe a significant improvement in both CLIP FID and Crop FID scores when the cross-attention neural operator is employed. The improved scores affirm the operator's usefulness in synthesizing high-resolution $4096\times4096$ images. The vanilla cross-attention only applies conditioning on a regular grid of coordinates, which cannot capture fine details between coarse grid points. 

\noindent\textbf{Computing Resource Evaluation.} We analyze the computing resources needed for training various image generation models on a single NVIDIA A100 40GB GPU. As detailed in Table~\ref{table:computing}, the training time and memory requirements for diffusion models in finite dimensions, such as SDXL, increase substantially when scaling from $1024 \times 1024$ to $4096 \times 4096$, making training infeasible on standard hardware. Conversely, the patch-based approach, while able to train at higher resolutions by dividing images into smaller patches, exhibits a parameter increase and reduced batch size. Our conditional diffusion model in function space maintains a consistent maximum batch size, significantly lower parameter count, and per epoch training time across resolutions ($12$ hours vs. $300$ hours and $140$ hours), demonstrating our method's superior scalability to sizes beyond the reach of existing methods.

\begin{table}[!t]
\centering
\caption{Computing resources requirements for different diffusion models. our \texttt{$\infty$-Brush} maintains a constant parameter count and batch size across resolutions, highlighting its efficiency and scalability for controllable large image generation.}
\footnotesize
\resizebox{1\linewidth}{!}{
\begin{tabular}{l|c|c|c|c|c}
\toprule
  \textbf{Method} & $\#$ Params. & \multicolumn{2}{c|}{Training at $1024\times 1024$} & \multicolumn{2}{c}{Training at $4096\times 4096$} \\ 
& & Max batch size & Epoch time &  Max batch size & Epoch time \\
\midrule
  SDXL \cite{podell2023sdxl} & $3.5$B &$4$ & $140$ hr & O.O.M & \makecell{$1000$ hr (estimated) \\ currently infeasible } \\
  \midrule
 Graikos \etal~\cite{graikos2023learned} & $860$M & $100$ & $45$ hr & $4$ & $300$ hr\\
 \midrule
 \texttt{$\infty$-Brush} & $450$M & $20$ & $12$ hr & $20$ & $12$ hr \\
\bottomrule
\end{tabular}
}
\label{table:computing}
\end{table}

\section{Limitations}
Although \texttt{$\infty$-Brush} images exhibit better global structure consistency and maintain a degree of fine detail, they are not better than other methods in terms of local details.
We highlight a few key reasons which we hypothesize hinder our model's performance.
Our model has the smallest parameter count, with just half the parameters of the model of \cite{graikos2023learned}. We expect model sizes to scale as more works focus on infinite-dimensional diffusion models and performance to increase, as was observed in regular, finite diffusion models. Additionally, both SDXL and \cite{graikos2023learned} utilize pre-trained models as initialization, whereas, ours is trained from scratch as no infinite-dimensional pre-trained models are available, leading to worse performance in smaller datasets.

\section{Conclusion}
In conclusion, \texttt{$\infty$-Brush} presents a necessary leap forward in the domain of conditional large image generation, particularly for applications demanding high-resolution and domain-specific conditional generation. This paper has demonstrated that our approach effectively addresses the scalability limitations inherent in previous diffusion models while retaining a high degree of control over the generated output. By proposing a novel conditional diffusion model in function space, complemented by a cross-attention neural operator, we achieve not only state-of-the-art fidelity in the global structure of the images but also maintain acceptable detail in higher-resolution images without the excessive computational costs typically associated with such tasks. In future work, we plan to design local neural operators to capture fine details and transfer knowledge from finite-dimensional diffusion models for powerful initialization.

\section*{Acknowledgments}
This research was partially supported by NCI awards 1R21CA258493-01A1, 5U24CA215109, UH3CA225021,  U24CA180924, NSF grants IIS-2123920, IIS-2212046, Stony Brook Profund 2022 seed funding, and generous support from Bob Beals and Betsy Barton.

\appendix


\section{Conditional Diffusion Models in Function Space} 
\label{appendix:proof}
\label{sec:conditional}
\textbf{Forward process.} The forward process of a conditional diffusion model in function space is defined as a discrete-time Markov chain that incrementally perturbs probability measure $\mathbb{Q}_{\mathrm{data}}$ towards a Gaussian measure $\mathcal{N}(\mathbf{m}, \mathbf{C})$ with a zero mean and a specified covariance operator $\mathbf{C}$. It is a time-indexed process where each step $\mathbf{u}_t$ is obtained by applying a transformation to the previous step $\mathbf{u}_{t-1}$, which involves a scaling factor $\sqrt{1-\beta_t}\mathbf{u}_{t-1}$ related to the variance schedule $\beta$, and adding scaled Gaussian noise $\sqrt{\beta_t}\bm{\xi}_t$ with $\bm{\xi}_t \sim \mathcal{N}(\mathbf{0}, \mathbf{C})$:
\begin{equation}
    \mathbf{u}_t = \sqrt{1-\beta_t}\mathbf{u}_{t-1} + \sqrt{\beta_t}\bm{\xi}_t \quad\quad t=1,2,\dots,T.
\end{equation}
Similar to diffusion models in finite dimensions, the forward process in function space also admits sampling $\mathbf{u}_t$ at an arbitrary timestep $t$ in closed form. For $\Bar{\alpha}_t = \prod_{i=1}^t(1-\beta_t)$, we have:
\begin{equation}
\begin{split}
    \mathbf{u}_t &= \sqrt{1-\beta_t}\mathbf{u}_{t-1} + \sqrt{\beta_t}\bm{\xi}_t \quad ; \textrm{where } \bm{\xi}_t, \bm{\xi}_{t-1},\dots \sim \mathcal{N}(\mathbf{0}, \mathbf{C}) \\
 &= \sqrt{(1-\beta_t)(1-\beta_{t-1})}\mathbf{u}_{t-2} + \sqrt{(1-\beta_t)\beta_{t-1}}\bm{\xi}_{t-1} + \sqrt{\beta_t}\bm{\xi}_t \quad \\
 &= \sqrt{(1-\beta_t)(1-\beta_{t-1})}\mathbf{u}_{t-2} + \sqrt{1 - \alpha_t\alpha_{t-1}}\Bar{\bm{\xi}}_{t-1} \quad \\
 &= \dots \\
 &= \sqrt{\Bar{\alpha}_t} \mathbf{u}_{0} + \sqrt{1 - \Bar{\alpha}_t}\bm{\xi}
\end{split}
\end{equation}

Based on the above analysis, we obtain:
\begin{equation}
    \mathbb{Q}\left(\mathbf{u}_t\vert\mathbf{u}_0\right) = \mathcal{N}\left(\mathbf{u}_t; \sqrt{\Bar{\alpha}_t}\mathbf{u}_0, (1-\Bar{\alpha}_t)\mathbf{C}\right).
\end{equation}

In the context of image generation, we discretize the function $\mathbf{u}_j$ on the mesh $\mathbf{x}_j = \{\mathbf{x}_j^{(i)}\}_{1 \le i \le N} \subset \mathcal{X}$ by sampling $N$ coordinates of each image, which results in a non-smooth input space. To achieve a smoother function representation, a smoothing operator \cite{hoogeboom2023blurring, rissanen2023generative} $\mathbf{A}: \mathcal{H} \rightarrow \mathcal{H}$, \eg a truncated Gaussian kernel, is applied to approximate the rough inputs within the function space $\mathcal{H}$:

\begin{equation}
    \mathbb{Q}\left(\mathbf{u}_t\vert\mathbf{u}_0\right) = \mathcal{N}\left(\mathbf{u}_t; \sqrt{\Bar{\alpha}_t}\mathbf{A}\mathbf{u}_0, (1-\Bar{\alpha}_t)\mathbf{A}\mathbf{C}\mathbf{A}^T\right). 
\end{equation}

\noindent\textbf{Reverse process.} The reverse process in the diffusion model framework is achieved by iteratively denoising from the Gaussian measure $\mathcal{N}(\mathbf{m},\mathbf{C})$ back towards the probability measure $\mathbb{Q}_0 = \mathbb{Q}_{\mathrm{data}}$. We use a variational approach to approximate posterior measures with a variational family of measures on $\mathcal{H}$ and incorporate the conditional embedding $\mathbf{e}$ to control the generation process. We model the underlying posterior measure $\mathbb{Q}(\mathbf{u}_{t-1}\vert\mathbf{u}_t)$ with a conditional Gaussian measure:
\begin{equation}
    \mathbb{P}_\theta(\mathbf{u}_{t-1}\vert\mathbf{u}_t, \mathbf{e}) = \mathcal{N}\left(\mathbf{u}_{t-1}; \mathbf{m}_\theta(\mathbf{u}_t, \mathbf{e}, t), \mathbf{A}\mathbf{C}_\theta(\mathbf{u}_t, \mathbf{e}, t)\mathbf{A}^T\right).
\end{equation}

Likewise, we are able to derive a closed-form representation of the forward process posteriors, which are tractable when conditioned on $\mathbf{u}_0$:
\begin{equation}
    \mathbb{Q}\left(\mathbf{u}_{t-1}\vert\mathbf{u}_t, \mathbf{u}_0\right) = \mathcal{N}\left(\mathbf{u}_{t-1};\mathbf{\tilde{m}}_t(\mathbf{u}_t, \mathbf{u}_0), \Tilde{\beta}_t\mathbf{C}\right).
\end{equation}

Using Bayes' rule, we obtain:
\begin{equation}
\resizebox{\textwidth}{!}{$
\begin{split}
& \mathbb{Q}(\mathbf{u}_{t-1} \vert \mathbf{u}_t, \mathbf{u}_0) = \mathbb{Q}(\mathbf{u}_t \vert \mathbf{u}_{t-1}, \mathbf{u}_0) \frac{ \mathbb{Q}(\mathbf{u}_{t-1} \vert \mathbf{u}_0) }{ \mathbb{Q}(\mathbf{u}_t \vert \mathbf{u}_0) } \\
&\propto \exp \Big(-\frac{1}{2} \big(\frac{\langle\mathbf{C}^{-1}(\mathbf{u}_t - \sqrt{\alpha_t} \mathbf{u}_{t-1}), \mathbf{u}_t - \sqrt{\alpha_t} \mathbf{u}_{t-1} \rangle}{\beta_t} + \frac{\langle \mathbf{C}^{-1} (\mathbf{u}_{t-1} - \sqrt{\bar{\alpha}_{t-1}} \mathbf{A}\mathbf{u}_0), \mathbf{u}_{t-1} - \sqrt{\bar{\alpha}_{t-1}} \mathbf{A}\mathbf{u}_0 \rangle}{1-\bar{\alpha}_{t-1}} \\
& - \frac{\langle \mathbf{C}^{-1}(\mathbf{u}_t - \sqrt{\bar{\alpha}_t}\mathbf{A}\mathbf{u}_0), \mathbf{u}_t - \sqrt{\bar{\alpha}_t} \mathbf{A}\mathbf{u}_0\rangle}{1-\bar{\alpha}_t} \big) \Big) \\
&= \exp\Big( -\frac{1}{2} \big((\frac{\alpha_t}{\beta_t} + \frac{1}{1 - \bar{\alpha}_{t-1}}) \langle \mathbf{C}^{-1}\mathbf{u}_{t-1},\mathbf{u}_{t-1}\rangle - 2\langle\mathbf{C}^{-1}(\frac{\sqrt{\alpha_t}}{\beta_t}  \mathbf{u}_t + \frac{\sqrt{\bar{\alpha}_{t-1}}}{1 - \bar{\alpha}_{t-1}} \mathbf{A} \mathbf{u}_0) ,\mathbf{u}_{t-1}\rangle + C(\mathbf{u}_t, \mathbf{u}_0) \big) \Big),
\end{split}$}
\end{equation}
where $C(\mathbf{u}_t, \mathbf{u}_0)$ is some function not involving $\mathbf{u}_{t-1}$ and details are omitted. Following the standard Gaussian density function, the mean and covariance of $\mathbb{Q}(\mathbf{u}_{t-1} \vert \mathbf{u}_t, \mathbf{u}_0)$ can be parameterized as follows (recall that $\alpha_t = 1 - \beta_t$ and $\bar{\alpha}_t = \prod_{i=1}^T \alpha_i$):
\begin{equation}
\begin{split}
\mathbf{\tilde{m}}_t (\mathbf{u}_t, \mathbf{u}_0)
&= (\frac{\sqrt{\alpha_t}}{\beta_t} \mathbf{u}_t + \frac{\sqrt{\bar{\alpha}_{t-1} }}{1 - \bar{\alpha}_{t-1}} \mathbf{A}\mathbf{u}_0)/(\frac{\alpha_t}{\beta_t} + \frac{1}{1 - \bar{\alpha}_{t-1}}) \\
&= (\frac{\sqrt{\alpha_t}}{\beta_t} \mathbf{u}_t + \frac{\sqrt{\bar{\alpha}_{t-1} }}{1 - \bar{\alpha}_{t-1}} \mathbf{A} \mathbf{u}_0) \frac{1 - \bar{\alpha}_{t-1}}{1 - \bar{\alpha}_t} \cdot \beta_t \\
&= \frac{\sqrt{\Bar{\alpha}_{t-1}}\beta_t}{1-\Bar{\alpha}_t} \mathbf{A}\mathbf{u}_0 + \frac{\sqrt{1-\beta_t}(1-\Bar{\alpha}_{t-1})}{1-\Bar{\alpha}_t}\mathbf{u}_t.\\
\end{split}
\label{eq:mu_}
\end{equation}

\begin{equation}
 \tilde{\beta}_t = 1/(\frac{\alpha_t}{\beta_t} + \frac{1}{1 - \bar{\alpha}_{t-1}}) 
= 1/(\frac{\alpha_t - \bar{\alpha}_t + \beta_t}{\beta_t(1 - \bar{\alpha}_{t-1})})
= \frac{1 - \bar{\alpha}_{t-1}}{1 - \bar{\alpha}_t} \cdot \beta_t.
\end{equation}

\begin{proposition}[Learning Objective]{The cross-entropy of conditional diffusion models in function space has a variational upper bound of}
\begin{align}
    \mathcal{L}_{\mathrm{CE}} = -\mathbb{E}_\mathbb{Q}\log{\mathbb{P}_\theta(\mathbf{u}_0\vert\mathbf{e})} &\le \mathbb{E}_\mathbb{Q} \Bigg[\underbrace{\mathrm{KL}(\mathbb{Q}(\mathbf{u}_T \vert \mathbf{u}_0) \parallel \mathbb{P}_\theta(\mathbf{u}_T))}_{\mathcal{L}_T} \underbrace{-\log \mathbb{P}_\theta(\mathbf{u}_0 \vert \mathbf{u}_1, \mathbf{e})}_{\mathcal{L}_0} \notag\\
&+ \sum_{t=2}^{T} \underbrace{\mathrm{KL}(\mathbb{Q}(\mathbf{u}_{t-1} \vert \mathbf{u}_{t}, \mathbf{u}_0) \parallel \mathbb{P}_\theta(\mathbf{u}_{t-1} \vert\mathbf{u}_{t}, \mathbf{e})}_{\mathcal{L}_{t-1}}\Bigg].
\end{align}
\label{proposition:objective}
\end{proposition}

\begin{proof}
The conditional diffusion model in function space is trained to minimize the cross entropy as the learning objective, which is equivalent to minimizing variational upper bound (VUB): 
\begin{equation}
\begin{split}
L_\text{CE} &= - \mathbb{E}_{\mathbb{Q}(\mathbf{u}_0\vert \mathbf{e})} \log \mathbb{P}_\theta(\mathbf{u}_0\vert \mathbf{e}) \\
&= - \mathbb{E}_{\mathbb{Q}(\mathbf{u}_0\vert \mathbf{e})} \log \Big( \int \mathbb{P}_\theta(\mathbf{u}_{0:T}\vert \mathbf{e}) d\mathbf{u}_{1:T} \Big) \\
&= - \mathbb{E}_{\mathbb{Q}(\mathbf{u}_0\vert \mathbf{e})} \log \Big( \int \mathbb{Q}(\mathbf{u}_{1:T} \vert \mathbf{u}_0, \mathbf{e}) \frac{\mathbb{P}_\theta(\mathbf{u}_{0:T}\vert \mathbf{e})}{\mathbb{Q}(\mathbf{u}_{1:T} \vert \mathbf{u}_{0}, \mathbf{e})} d\mathbf{u}_{1:T} \Big) \\
&= - \mathbb{E}_{\mathbb{Q}(\mathbf{u}_0\vert \mathbf{e})} \log \Big( \mathbb{E}_{\mathbb{Q}(\mathbf{u}_{1:T} \vert \mathbf{u}_0, \mathbf{e})} \frac{\mathbb{P}_\theta(\mathbf{u}_{0:T}\vert \mathbf{e})}{\mathbb{Q}(\mathbf{u}_{1:T} \vert \mathbf{u}_{0}, \mathbf{e})} \Big) \\
&\leq - \mathbb{E}_{\mathbb{Q}(\mathbf{u}_{0:T}\vert \mathbf{e})} \log \frac{\mathbb{P}_\theta(\mathbf{u}_{0:T}\vert \mathbf{e})}{\mathbb{Q}(\mathbf{u}_{1:T} \vert \mathbf{u}_{0}, \mathbf{e})} \\
&= \mathbb{E}_{\mathbb{Q}(\mathbf{u}_{0:T}\vert \mathbf{e})}\Big[\log \frac{\mathbb{Q}(\mathbf{u}_{1:T} \vert \mathbf{u}_{0}, \mathbf{e})}{\mathbb{P}_\theta(\mathbf{u}_{0:T}\vert \mathbf{e})} \Big] \\
&= \mathbb{E}_{\mathbb{Q}(\mathbf{u}_{0:T}\vert \mathbf{e})}\Big[\log \frac{\mathbb{Q}(\mathbf{u}_{1:T} \vert \mathbf{u}_{0})}{\mathbb{P}_\theta(\mathbf{u}_{0:T}\vert \mathbf{e})} \Big] = L_\mathrm{VUB}.
\end{split}
\label{eq:lce}
\end{equation}

To convert each term in the equation to be analytically computable, the objective can be further rewritten to be a combination of several KL-divergence and entropy terms:
\begin{equation}
\resizebox{\textwidth}{!}{$
\begin{split}
L_\text{VUB} &= \mathbb{E}_{\mathbb{Q}(\mathbf{u}_{0:T}\vert \mathbf{e})} \Big[ \log\frac{\mathbb{Q}(\mathbf{u}_{1:T}\vert\mathbf{u}_0)}{\mathbb{P}_\theta(\mathbf{u}_{0:T}\vert \mathbf{e})} \Big] \\
&= \mathbb{E}_\mathbb{Q} \Big[ \log\frac{\prod_{t=1}^T \mathbb{Q}(\mathbf{u}_t\vert\mathbf{u}_{t-1})}{ \mathbb{P}_\theta(\mathbf{u}_T) \prod_{t=1}^T \mathbb{P}_\theta(\mathbf{u}_{t-1} \vert\mathbf{u}_t, \mathbf{e})} \Big] \\
&= \mathbb{E}_\mathbb{Q} \Big[ -\log \mathbb{P}_\theta(\mathbf{u}_T) + \sum_{t=1}^T \log \frac{\mathbb{Q}(\mathbf{u}_t\vert\mathbf{u}_{t-1})}{\mathbb{P}_\theta(\mathbf{u}_{t-1} \vert\mathbf{u}_t, \mathbf{e})} \Big] \\
&= \mathbb{E}_\mathbb{Q} \Big[ -\log \mathbb{P}_\theta(\mathbf{u}_T) + \sum_{t=2}^T \log \frac{\mathbb{Q}(\mathbf{u}_t\vert\mathbf{u}_{t-1})}{\mathbb{P}_\theta(\mathbf{u}_{t-1} \vert\mathbf{u}_t, \mathbf{e})} + \log\frac{\mathbb{Q}(\mathbf{u}_1 \vert \mathbf{u}_0)}{\mathbb{P}_\theta(\mathbf{u}_0 \vert \mathbf{u}_1, \mathbf{e})} \Big] \\
&= \mathbb{E}_\mathbb{Q} \Big[ -\log \mathbb{P}_\theta(\mathbf{u}_T) + \sum_{t=2}^T \log \Big( \frac{\mathbb{Q}(\mathbf{u}_{t-1} \vert \mathbf{u}_t, \mathbf{u}_0)}{\mathbb{P}_\theta(\mathbf{u}_{t-1} \vert\mathbf{u}_t, \mathbf{e})}\cdot \frac{\mathbb{Q}(\mathbf{u}_t \vert \mathbf{u}_0)}{\mathbb{Q}(\mathbf{u}_{t-1}\vert\mathbf{u}_0)} \Big) + \log \frac{\mathbb{Q}(\mathbf{u}_1 \vert \mathbf{u}_0)}{\mathbb{P}_\theta(\mathbf{u}_0 \vert \mathbf{u}_1, \mathbf{e})} \Big] \\
&= \mathbb{E}_\mathbb{Q} \Big[ -\log \mathbb{P}_\theta(\mathbf{u}_T) + \sum_{t=2}^T \log \frac{\mathbb{Q}(\mathbf{u}_{t-1} \vert \mathbf{u}_t, \mathbf{u}_0)}{\mathbb{P}_\theta(\mathbf{u}_{t-1} \vert\mathbf{u}_t, \mathbf{e})} + \sum_{t=2}^T \log \frac{\mathbb{Q}(\mathbf{u}_t \vert \mathbf{u}_0)}{\mathbb{Q}(\mathbf{u}_{t-1} \vert \mathbf{u}_0)} + \log\frac{\mathbb{Q}(\mathbf{u}_1 \vert \mathbf{u}_0)}{\mathbb{P}_\theta(\mathbf{u}_0 \vert \mathbf{u}_1, \mathbf{e})} \Big] \\
&= \mathbb{E}_\mathbb{Q} \Big[ -\log \mathbb{P}_\theta(\mathbf{u}_T) + \sum_{t=2}^T \log \frac{\mathbb{Q}(\mathbf{u}_{t-1} \vert \mathbf{u}_t, \mathbf{u}_0)}{\mathbb{P}_\theta(\mathbf{u}_{t-1} \vert\mathbf{u}_t, \mathbf{e})} + \log\frac{\mathbb{Q}(\mathbf{u}_T \vert \mathbf{u}_0)}{\mathbb{Q}(\mathbf{u}_1 \vert \mathbf{u}_0)} + \log \frac{\mathbb{Q}(\mathbf{u}_1 \vert \mathbf{u}_0)}{\mathbb{P}_\theta(\mathbf{u}_0 \vert \mathbf{u}_1, \mathbf{e})} \Big]\\
&= \mathbb{E}_\mathbb{Q} \Big[ \log\frac{\mathbb{Q}(\mathbf{u}_T \vert \mathbf{u}_0)}{\mathbb{P}_\theta(\mathbf{u}_T)} + \sum_{t=2}^T \log \frac{\mathbb{Q}(\mathbf{u}_{t-1} \vert \mathbf{u}_t, \mathbf{u}_0)}{\mathbb{P}_\theta(\mathbf{u}_{t-1} \vert\mathbf{u}_t, \mathbf{e})} - \log \mathbb{P}_\theta(\mathbf{u}_0 \vert \mathbf{u}_1, \mathbf{e}) \Big] \\
&= \mathbb{E}_\mathbb{Q} [\underbrace{\text{KL}(\mathbb{Q}(\mathbf{u}_T \vert \mathbf{u}_0) \parallel \mathbb{P}_\theta(\mathbf{u}_T))}_{L_T} + \sum_{t=2}^{T} \underbrace{\text{KL}(\mathbb{Q}(\mathbf{u}_{t-1} \vert \mathbf{u}_{t}, \mathbf{u}_0) \parallel \mathbb{P}_\theta(\mathbf{u}_{t-1} \vert\mathbf{u}_{t}, \mathbf{e}))}_{L_{t-1}} \underbrace{- \log \mathbb{P}_\theta(\mathbf{u}_0 \vert \mathbf{u}_1, \mathbf{e})}_{L_0} ]
\end{split}$}
\label{eq:vub}
\end{equation}
Combine Eq.~\ref{eq:lce} and Eq.~\ref{eq:vub}, we obtain:
\begin{align}
    \mathcal{L}_{\mathrm{CE}} = -\mathbb{E}_\mathbb{Q}\log{\mathbb{P}_\theta(\mathbf{u}_0\vert\mathbf{e})} &\le \mathbb{E}_\mathbb{Q} \Bigg[\underbrace{\mathrm{KL}(\mathbb{Q}(\mathbf{u}_T \vert \mathbf{u}_0) \parallel \mathbb{P}_\theta(\mathbf{u}_T))}_{\mathcal{L}_T} \underbrace{-\log \mathbb{P}_\theta(\mathbf{u}_0 \vert \mathbf{u}_1, \mathbf{e})}_{\mathcal{L}_0} \notag\\
&+ \sum_{t=2}^{T} \underbrace{\mathrm{KL}(\mathbb{Q}(\mathbf{u}_{t-1} \vert \mathbf{u}_{t}, \mathbf{u}_0) \parallel \mathbb{P}_\theta(\mathbf{u}_{t-1} \vert\mathbf{u}_{t}, \mathbf{e})}_{\mathcal{L}_{t-1}}\Bigg].
\end{align}
\qed
\end{proof}
To compute the KL divergence between probability measures $\mathrm{KL}(\mathbb{Q} \parallel \mathbb{P})$, we need to utilize a measure-theoretic definition of the KL divergence, which is stated in the following lemmas \cite{DaPrato_Zabczyk_2014}. 

\begin{lemma}[Measure Equivalence - The Feldman-Hájek Theorem]
    Let $\mathbb{Q} = \mathcal{N}(\mathbf{m}_1,\mathbf{C}_1)$ and $\mathbb{P} = \mathcal{N}(\mathbf{m}_2,\mathbf{C}_2)$ be Gaussian measures on $\mathcal{H}$. They are equivalent if and only if $(i): \mathbf{C}^{1/2}_1(\mathcal{H}) = \mathbf{C}^{1/2}_2(\mathcal{H}) = \mathcal{H}_0$, $(ii): \mathbf{m}_1 - \mathbf{m}_2 \in \mathcal{H}_0$, and $(iii):$ The operator $(\mathbf{C}^{-1/2}_1\mathbf{C}^{1/2}_2)(\mathbf{C}^{-1/2}_1\mathbf{C}^{1/2}_2)^* -\mathbf{I}$ is a Hilbert-Schmidt operator on the closure $\overline{\mathcal{H}_0}$.
\label{lemma:measure_equi}
\end{lemma}
\begin{proof}
Refer to the proof of Theorem 2.25 of Da Prato and Zabczyk \cite{DaPrato_Zabczyk_2014}.\qed
\end{proof}

\begin{lemma}[The Radon-Nikodym Derivative]
Let $\mathbb{Q} = \mathcal{N}(\mathbf{m}_1,\mathbf{C}_1)$ and $\mathbb{P} = \mathcal{N}(\mathbf{m}_2,\mathbf{C}_2)$ be Gaussian measures on $\mathcal{H}$. If $\mathbb{P}$ and $\mathbb{Q}$ are equivalent and $\mathbf{C}_1 = \mathbf{C}_2 = \mathbf{C}$, then $\mathbb{P}$-a.s. the Radon-Nikodym derivative $\mathrm{d}\mathbb{Q}/\mathrm{d}\mathbb{P}$ is given by
\begin{equation}    
\resizebox{\textwidth}{!}{
$\frac{\mathrm{d} \mathbb{Q}}{\mathrm{d} \mathbb{P}}(\mathbf{f})=\exp \left[\left\langle \mathbf{C}^{-1 / 2}\left(\mathbf{m}_1-\mathbf{m}_2\right), \mathbf{C}^{-1 / 2}\left(\mathbf{f}-\mathbf{m}_2\right)\right\rangle-\frac{1}{2} \| \mathbf{C}^{-1 / 2} (\mathbf{m}_1-\mathbf{m}_2) \|^2\right] \forall \mathbf{f} \in \mathcal{H}.$
} 
\end{equation}
\label{lemma:radon}
\end{lemma}
\begin{proof}
Refer to the proof of Theorem 2.23 of Da Prato and Zabczyk \cite{DaPrato_Zabczyk_2014}.\qed
\end{proof}

Lemma~\ref{lemma:measure_equi} states the three conditions under which two Gaussian measures are equivalent. Lemma~\ref{lemma:radon} is the consequence of the Feldman-Hájek theorem, providing the Radon-Nikodym derivative formula when dealing with Gaussian measures on $\mathcal{H}$.

To train the diffusion model in functional space we have to minimize the upper bound of Proposition~\ref{proposition:objective}, which requires us to compute the KL divergence between the measures $\mathbb{Q}, \mathbb{P}$. In order to satisfy Lemma~\ref{lemma:measure_equi}, which will enable us to use Lemma~\ref{lemma:radon} to compute the KL divergence, we make the following assumption:
\begin{assumption}
Let $\mathbb{Q} = \mathcal{N}(\mathbf{\tilde{m}}_t(\mathbf{u}_t, \mathbf{u}_0), \Tilde{\beta}_t\mathbf{C})$ and $\mathbb{P}_\theta = \mathcal{N}(\mathbf{m}_\theta(\mathbf{u}_t, \mathbf{e}, t), \Tilde{\beta}_t\mathbf{C})$ be Gaussian measures on $\mathcal{H}$. With a conditional component $\mathbf{e}$, which can be an element of finite-dimensional space $\mathbb{R}^d$ or Hilbert space $\mathcal{H}$, there exists a parameter set $\theta$ such that the difference in mean elements of the two measures falls within the scaled covariance space: 
\label{assumption:measure_equi}
\begin{equation}
\mathbf{\tilde{m}}_t(\mathbf{u}_t, \mathbf{u}_0)-\mathbf{m}_\theta(\mathbf{u}_t, \mathbf{e}, t) \in (\Tilde{\beta}_t\mathbf{C})^{1/2}(\mathcal{H}).    
\end{equation}
\end{assumption}

By making this assumption we satisfy all three conditions of Lemma~\ref{lemma:measure_equi}:
$(i): \mathbf{C}^{1/2}_1(\mathcal{H}) = \mathbf{C}^{1/2}_2(\mathcal{H}) = (\Tilde{\beta}_t\mathbf{C})^{1/2}(\mathcal{H}) = \mathcal{H}_0$; 
$(ii):  \mathbf{m}_1 - \mathbf{m}_2 \in \mathcal{H}_0$  is directly satisfied from Assumption~\ref{assumption:measure_equi}; 
$(iii): (\mathbf{C}^{-1/2}_1\mathbf{C}^{1/2}_2)(\mathbf{C}^{-1/2}_1\mathbf{C}^{1/2}_2)^* -\mathbf{I} = \mathbf{I} - \mathbf{I}$ is the zero operator, which is trivially a Hilbert-Schmidt operator as its Hilbert-Schmidt norm is $0$.
As a consequence, $\mathbb{Q}$ and $\mathbb{P}$ are equivalent, allowing us to utilize the Radon-Nikodym derivative from Lemma~\ref{lemma:radon}.

\begin{theorem}[Conditional Diffusion Optimality in Function Space]\\
    Given the specified conditions in Assumption~\ref{assumption:measure_equi}, the minimization of the learning objective in Proposition~\ref{proposition:objective} is equivalent to obtaining the parameter set $\theta^*$ that is the solution to the problem
    {\small
    \begin{equation}
       \theta^* = \argmin_\theta \mathbb{E}_{\mathbf{u}_0\sim\mathbb{Q}_{\mathrm{data}}, t\sim[1, T]} \lambda_t\left\vert\left\vert \mathbf{C}^{-1/2}\left(\mathbf{A}\bm{\xi} - \bm{\xi}_\theta(\sqrt{\Bar{\alpha}_t}\mathbf{A}\mathbf{u}_0 + \sqrt{1-\Bar{\alpha}_t}\mathbf{A}\bm{\xi}, \mathbf{e}
       ,t)\right)\right\vert\right\vert_\mathcal{H}^2,
    \end{equation}}
where $\bm{\xi}\sim\mathcal{N}(\mathbf{0}, \mathbf{C})$, $\mathbf{A}: \mathcal{H}\rightarrow\mathcal{H}$ denotes a smoothing operator, $\mathbf{e} \in (\mathbb{R}^d \cup \mathcal{H})$ is a conditional component, $\bm{\xi}_\theta : \{1, 2, \dots, T\} \times (\mathbb{R}^d \cup \mathcal{H}) \times \mathcal{H} \rightarrow \mathcal{H}$ is a parameterized mapping, $\lambda_t = \beta^2_t/2\Tilde{\beta}_t(1-\beta_t)(1-\Bar{\alpha}_t) \in \mathbb{R}$ is a time-dependent constant.
\label{theorem:optimality}
\end{theorem}
\begin{proof}
   Under Assumption~\ref{assumption:measure_equi}, we are now able to use the Radon-Nikodym derivative to compute the KL divergence: 
\begin{equation}
\resizebox{\textwidth}{!}{$
\begin{split} 
\mathrm{KL} \left[\mathbb{Q} \parallel \mathbb{P}\right] &= \int_\mathcal{H}\log\frac{\mathrm{d}\mathbb{Q}}{\mathrm{d}\mathbb{P}}(\mathbf{f})\mathrm{~d}\mathbb{Q}(\mathbf{f}) \\
 &= -\frac{1}{2} \| \mathbf{C}^{-1 / 2} (\mathbf{m}_1-\mathbf{m}_2) \|^2_\mathcal{H} + \int_\mathcal{H}\left\langle \mathbf{C}^{-1 / 2}\left(\mathbf{m}_1-\mathbf{m}_2\right), \mathbf{C}^{-1 / 2}\left(\mathbf{f}-\mathbf{m}_2\right)\right\rangle \mathrm{~d}\mathbb{Q}(\mathbf{f}).
 \end{split}
 $}
 \label{eq:kl_divergence}
 \end{equation}
We now use spectral decomposition to compute the integral term. Let $\{\lambda_j, \mathbf{e}_j\}_{j=1}^\infty$ be the eigenvalues and eigenvectors of $\mathbf{C}$. The eigenvector of $\mathbf{C}$ form an orthonormal basis for $\mathcal{H}$ by the spectral theorem, as $\mathbf{C}$ is a self-adjoint compact operator. Hence, the second integral is:
\begin{equation}
\begin{split}
& \int_{\mathcal{H}}\left\langle \mathbf{C}^{-1 / 2}\left(\mathbf{m}_1-\mathbf{m}_2\right), \mathbf{C}^{-1 / 2}\left(\mathbf{f}-\mathbf{m}_2\right) \mathrm{~d} \mathbb{Q}(f)\right. \\
& =\int_{\mathcal{H}} \sum_{j=1}^{\infty}\left\langle \mathbf{m}_1-\mathbf{m}_2, \mathbf{e}_j\right\rangle\left\langle \mathbf{f}-\mathbf{m}_2, \mathbf{e}_j\right\rangle \lambda_j^{-1} \mathrm{~d} \mathbb{Q}(f) \\
& =\sum_{j=1}^{\infty} \lambda_j^{-1}\left\langle \mathbf{m}_1-\mathbf{m}_2, \mathbf{e}_j\right\rangle \int_{\mathcal{H}}\left\langle \mathbf{f}-\mathbf{m}_2, \mathbf{e}_j\right\rangle \mathrm{~d} \mathbb{Q}(f) \\
& =\sum_{j=1}^{\infty} \lambda_j^{-1}\left\langle \mathbf{m}_1-\mathbf{m}_2, \mathbf{e}_j\right\rangle^2 \\
& =\left\langle \mathbf{C}^{-1 / 2}\left(\mathbf{m}_1-\mathbf{m}_2\right), \mathbf{C}^{-1 / 2}\left(\mathbf{m}_1-\mathbf{m}_2\right)\right\rangle .
\end{split}    
\label{eq:spectral_decompose}
\end{equation}
Combine Eq.~\ref{eq:kl_divergence} and Eq.~\ref{eq:spectral_decompose}, we obtain:
\begin{equation}
\mathrm{KL} \left[\mathbb{Q} \parallel \mathbb{P}\right] = \frac{1}{2} \| \mathbf{C}^{-1 / 2} (\mathbf{m}_1-\mathbf{m}_2) \|^2_\mathcal{H}
\end{equation}
From Proposition~\ref{proposition:objective}, the KL divergence between Gaussian measures $\mathbb{Q}$ and $\mathbb{P}$ now becomes:
\begin{equation}
\begin{split}
    L_{t-1} &= \mathrm{KL}\left[\mathbb{Q}(\mathbf{u}_{t-1} \vert \mathbf{u}_{t}, \mathbf{u}_0) \parallel \mathbb{P}_\theta(\mathbf{u}_{t-1} \vert\mathbf{u}_{t}, \mathbf{e})\right] \\
    &= \frac{1}{2} \| (\Tilde{\beta}_t\mathbf{C})^{-1 / 2} (\mathbf{\tilde{m}}_t (\mathbf{u}_t, \mathbf{u}_0)-\mathbf{m}_\theta(\mathbf{u}_t, \mathbf{e}, t)) \|^2_\mathcal{H}
\end{split}
\end{equation}
Our model must predict the mean function $\mathbf{\tilde{m}}_t (\mathbf{u}_t, \mathbf{u}_0)$. Recall that we got the expression of $\mathbf{\tilde{m}}_t (\mathbf{u}_t, \mathbf{u}_0)$ and $\mathbf{u}_0$ depending on $\mathbf{u}_t$:
\begin{equation}
\mathbf{\tilde{m}}_t (\mathbf{u}_t, \mathbf{u}_0) = \frac{\sqrt{\Bar{\alpha}_{t-1}}\beta_t}{1-\Bar{\alpha}_t} \mathbf{A}\mathbf{u}_0 + \frac{\sqrt{1-\beta_t}(1-\Bar{\alpha}_{t-1})}{1-\Bar{\alpha}_t}\mathbf{u}_t.
\end{equation}
\begin{equation}
    \mathbf{Au}_0 = \frac{1}{\sqrt{\Bar{\alpha}_t}} \left(\mathbf{u}_{t} - \sqrt{1 - \Bar{\alpha}_t}\mathbf{A}\bm{\xi}\right) \quad ;\text{where } \bm{\xi} \sim \mathcal{N}(\mathbf{0}, \mathbf{C})
\end{equation}
Combine these two expressions, we have:
\begin{equation}
\mathbf{\tilde{m}}_t (\mathbf{u}_t, \mathbf{u}_0) = \frac{1}{\sqrt{1-\beta_t}} \left(\mathbf{u}_t - \frac{\beta_t}{\sqrt{1-\Bar{\alpha}_t}}\mathbf{A}\bm{\xi}\right)   
\label{eq:mean_true}
\end{equation}
Thus, we parameterize the variational mean via:
\begin{equation}
    \mathbf{m}_\theta(\mathbf{u}_t, \mathbf{e}, t) =\frac{1}{\sqrt{1-\beta_t}} \left(\mathbf{u}_t - \frac{\beta_t}{\sqrt{1-\Bar{\alpha}_t}}\bm{\xi}_\theta(\mathbf{u}_t, \mathbf{e}, t)\right)
\label{eq:para_mean}
\end{equation}
Finally, plugging Eq.~\ref{eq:mean_true} and Eq.~\ref{eq:para_mean} into $L_{t-1}$, we obtain:
\begin{equation}
\resizebox{\textwidth}{!}{$
\begin{split}
    L_{t-1} &= \frac{1}{2} \left\| (\Tilde{\beta}_t\mathbf{C})^{-1 / 2} \left(\frac{1}{\sqrt{1-\beta_t}} \frac{\beta_t}{\sqrt{1-\Bar{\alpha}_t}}\mathbf{A}\bm{\xi} -\frac{1}{\sqrt{1-\beta_t}} \frac{\beta_t}{\sqrt{1-\Bar{\alpha}_t}}\bm{\xi}_\theta(\mathbf{u}_t, \mathbf{e}, t)\right) \right\|^2_\mathcal{H} \\
    &= \frac{\beta_t^2}{2\Tilde{\beta}_t(1-\beta_t)(1-\Bar{\alpha}_t)}\left\|\mathbf{C}^{-1 / 2}\left(\mathbf{A}\bm{\xi} -\bm{\xi}_\theta(\mathbf{u}_t, \mathbf{e}, t) \right)\right\|^2_\mathcal{H} \\
    & = \frac{\beta_t^2}{2\Tilde{\beta}_t(1-\beta_t)(1-\Bar{\alpha}_t)}\left\|\mathbf{C}^{-1 / 2}\left(\mathbf{A}\bm{\xi} -\bm{\xi}_\theta(\sqrt{\Bar{\alpha}_t}\mathbf{A}\mathbf{u}_0 + \sqrt{1-\Bar{\alpha}_t}\mathbf{A}\bm{\xi}, \mathbf{e}, t) \right)\right\|^2_\mathcal{H}
\end{split}
$}
\end{equation}
\qed
\end{proof}

\section{Experiments}
\label{appendix:experiments}
\subsection{Long-range dependencies}
We obtained the patch-based large-image model of Graikos et al. \cite{graikos2023learned} directly from the authors and tried to apply it to synthesize images larger than $1024\times 1024$ pixels. The overreliance of the patch-based model on the local descriptors (patch SSL embeddings) leads to the loss of large-scale structures and fails to capture long-range dependencies across the image. As a qualitative example (Figure~\ref{fig:long_range}), we get a reference image of size $2048 \times 2048$ pixels from TCGA-BRCA and extract embeddings in an attempt to generate a variation of it using our model and the patch-based model of \cite{graikos2023learned}. As illustrated, \texttt{$\infty$-Brush} retains large-scale structures (such as clearly-separated clusters of cells) that can span multiple patches, in comparison to the image generated from \cite{graikos2023learned}.

\begin{figure}[!t]
    \centering
    \includegraphics[width=\linewidth]{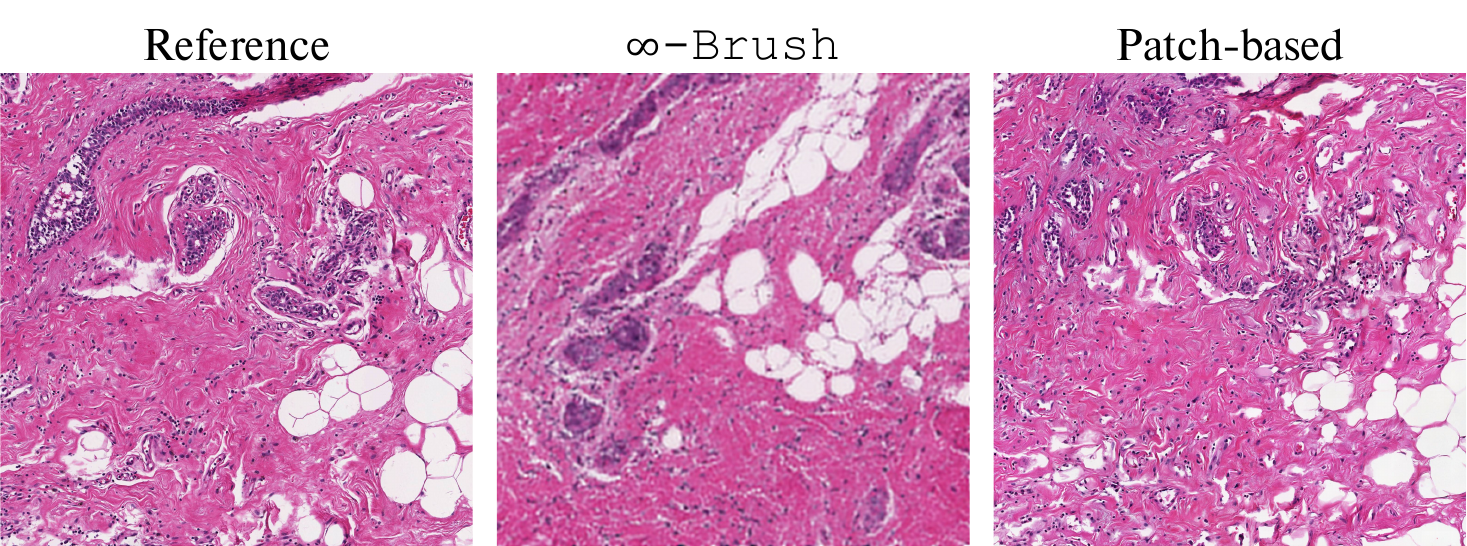}
    \caption{Long-range dependencies comparison between our \texttt{$\infty$-Brush} and patched-based method \cite{graikos2023learned}. \texttt{$\infty$-Brush} retains large-scale structures (such as clearly-separated clusters of cells) that can span multiple patches in comparison to the image generated from~\cite{graikos2023learned}.}
    \label{fig:long_range}
\end{figure}

\subsection{Zero-shot classification}
Following the experiment of \cite{graikos2023learned}, we generate images from a pre-defined set of four classes: benign tissue, in-situ, invasive carcinoma, and normal tissue. We use a VLM (Quilt) as a zero-shot classifier and compute the confusion matrix (CM). Figure~\ref{fig:cm_quilt} shows that \texttt{$\infty$-Brush} generates images semantically aligned with the text prompts.

\begin{figure}[!t]
    \centering
    \includegraphics[width=0.7\columnwidth]{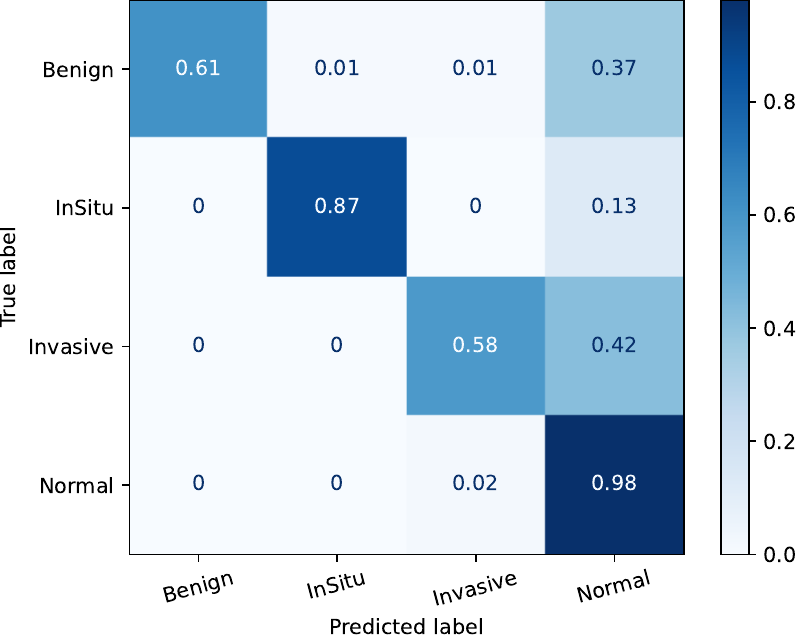}
    \caption{Confusion matrix of zero-shot classification of generated images.}
    \label{fig:cm_quilt}
\end{figure}

\subsection{Application of synthetic data on downstream task}
As a practical application, we double the number of training images of the BACH dataset by synthesizing images using real data embedding and evaluating the test set. Table~\ref{tab:bach_aug} shows a significant accuracy boost from these synthetic images.

\begin{table}[!t]
\caption{Synthetic data improves the accuracy significantly on the BACH test set.}
\centering
\resizebox{0.4\linewidth}{!}{
    \begin{tabular}{|c|c|}
    \hline
    Training Data       &  Test Acc         \\ 
    \hline
    Real             & 79 \%            \\ 
    Real + synthetic & \textbf{83 \%}  \\ 
    \hline
    \end{tabular}
}
\label{tab:bach_aug}
\end{table}

\subsection{Ablation study on $\%$ of pixels for training}

We compare our model when training on $256 ^* 256$ ($0.4\%$) \emph{vs.} $512 ^* 512$ pixels ($1.6\%$). Figure~\ref{fig:pixels_ablation} shows that training with more pixels improves performance. Our model efficiently uses $0.4\%$ of pixels compared to $25\%$ of $\infty$-Diff's due to the incorporation of coordinate embedding in CANO, functioning as positional embedding. 
\begin{figure}[t!]
    \centering
    \includegraphics[width=\columnwidth]{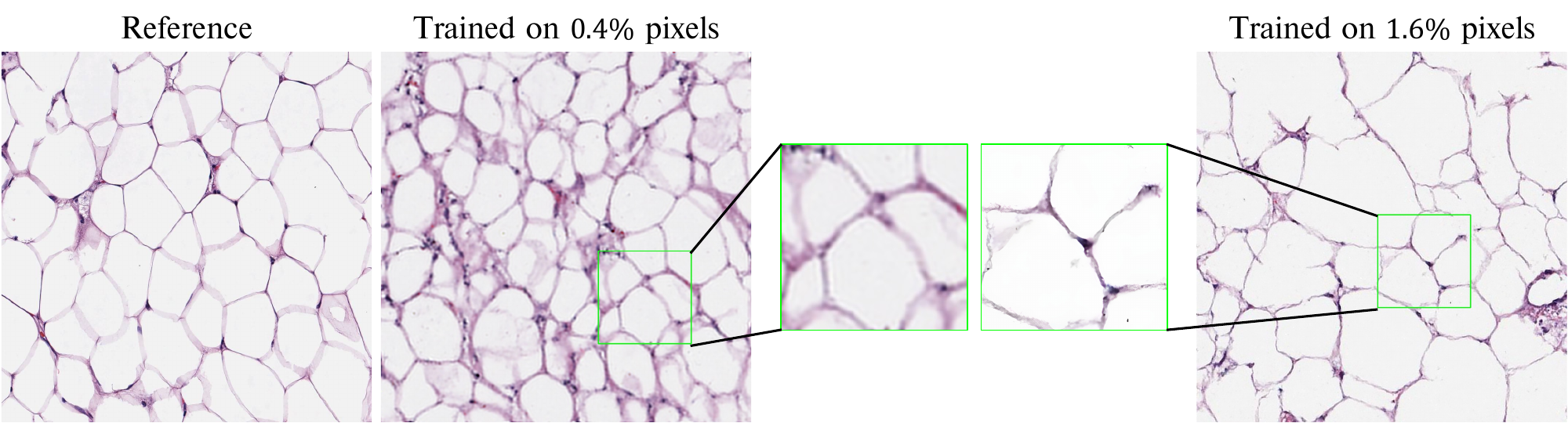}
    \caption{Ablation on $\%$ pixels for training and zoomed-in views.}
    \label{fig:pixels_ablation}
\end{figure}
\subsection{Qualitative results}
In Figure~\ref{fig:very_large_image} and Figure~\ref{fig:large_image}, we illustrate the generated very large ($4096 \times 4096$) and large ($1024 \times 1024$) images of TCGA-BRCA \cite{cancer2013cancer} dataset. We also show synthesized satellite images at $2048 \times 2048$ and $1024 \times 1024$ resolutions in Figure~\ref{fig:large_image_satellite}. Qualitative results show that given a single embedding vector of a downsampled $256\times256$ real image, \texttt{$\infty$-Brush} can synthesize images of arbitrary resolutions up to $4096 \times 4096$ and preserve global structures of the reference image. 

Figure~\ref{fig:failures} shows examples where the model did not successfully capture spatial structures and details from the reference images. This can be attributed to both the model and the conditioning used to represent the images.

\begin{figure}[!t]
    \centering
    \includegraphics[width=\linewidth]{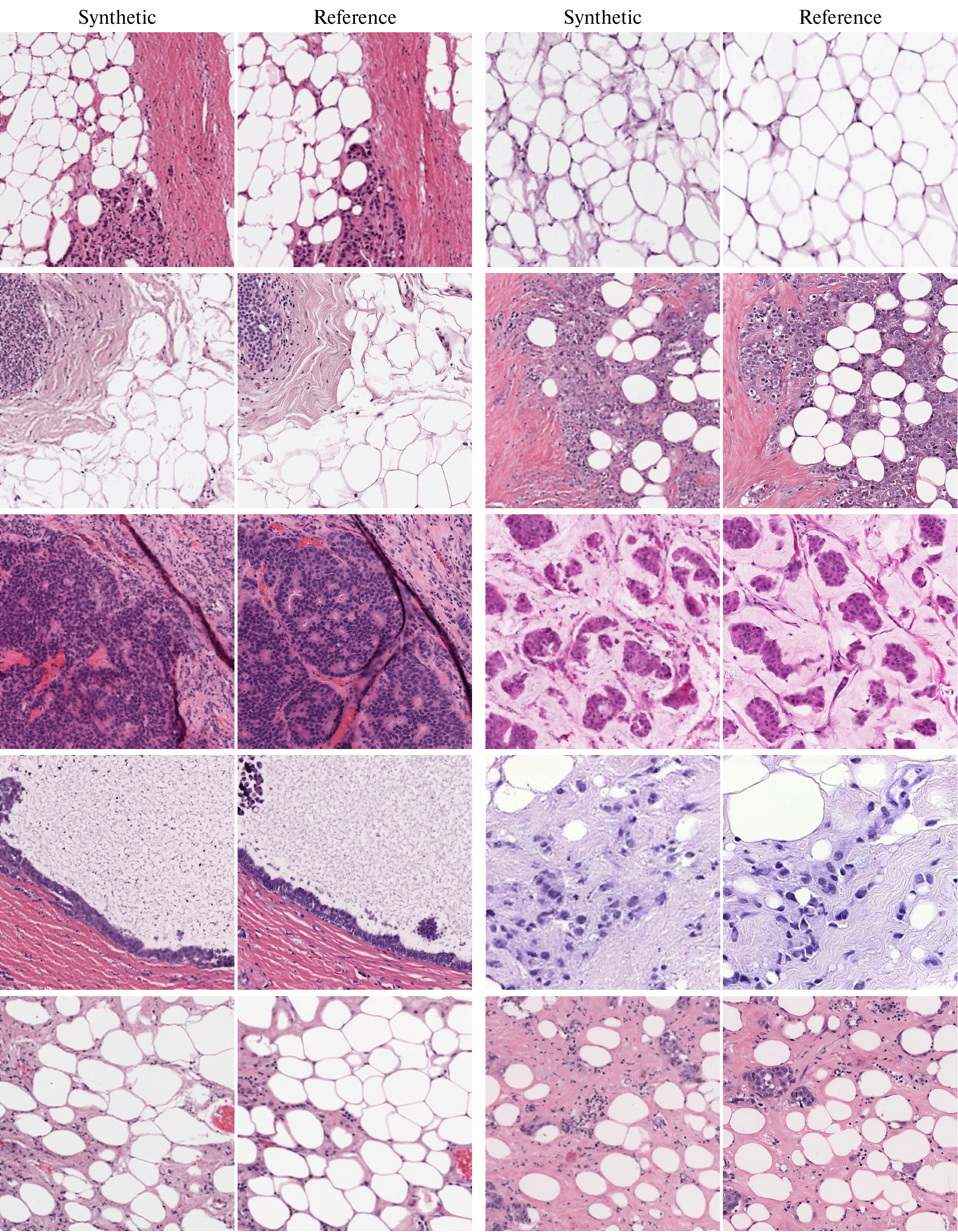}
    \caption{Very large ($4096 \times 4096$) images generated from \texttt{$\infty$-Brush}, and the corresponding reference real images used to generate them. Given a single embedding vector of a downsampled $256\times256$ real image, \texttt{$\infty$-Brush} can synthesize images of up to $4096 \times 4096$ and preserve global structures of the reference image.}
    \label{fig:very_large_image}
\end{figure}

\begin{figure}[!t]
    \centering
    \includegraphics[width=\linewidth]{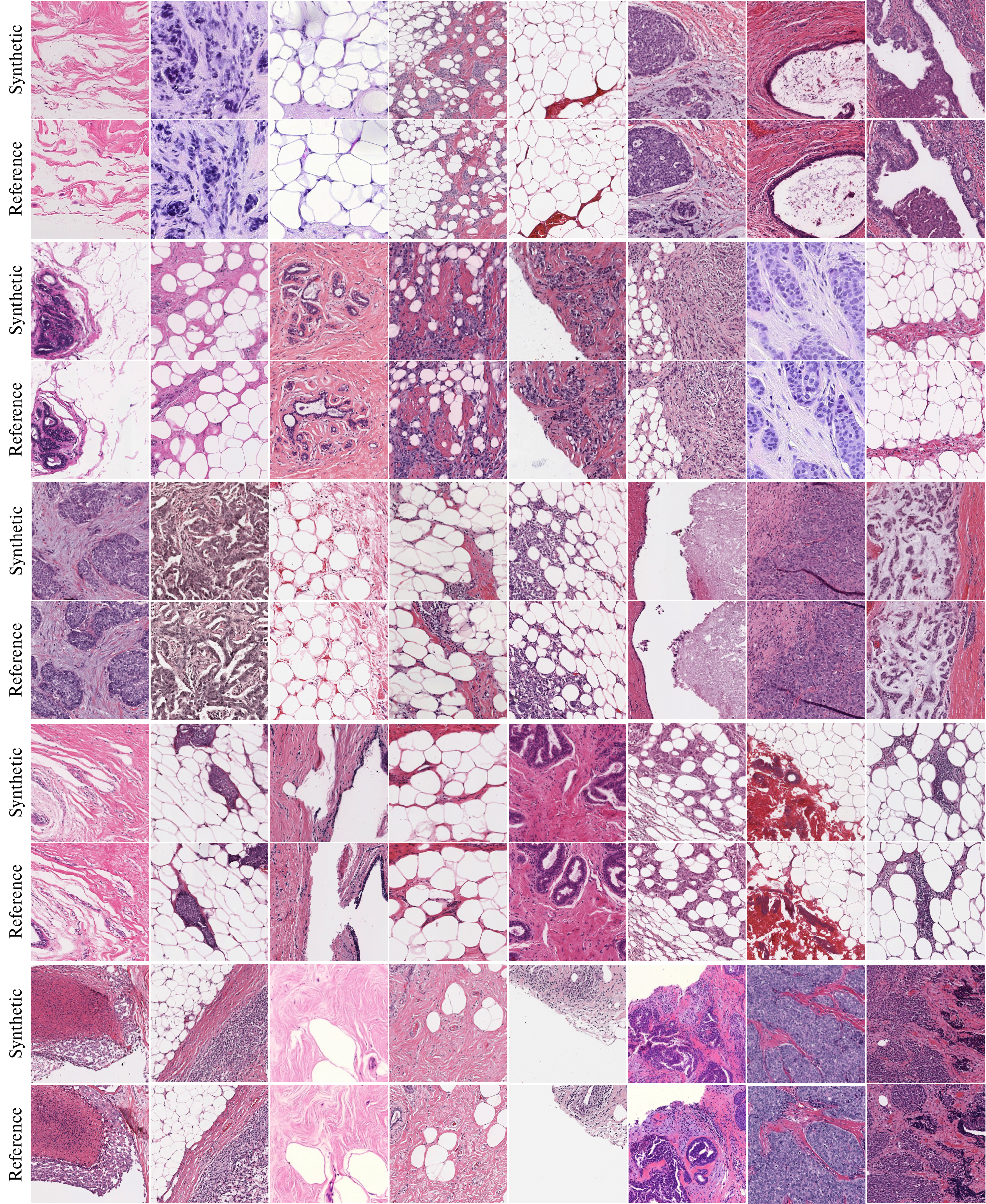}
    \caption{Large ($1024 \times 1024$) images generated from \texttt{$\infty$-Brush}, and the corresponding reference real images used to generate them. Given a single embedding vector of a downsampled $256\times256$ real image, \texttt{$\infty$-Brush} can synthesize images at arbitrary resolutions and preserve global structures of the reference image.}
    \label{fig:large_image}
\end{figure}

\begin{figure}[!t]
    \centering
    \includegraphics[width=\linewidth]{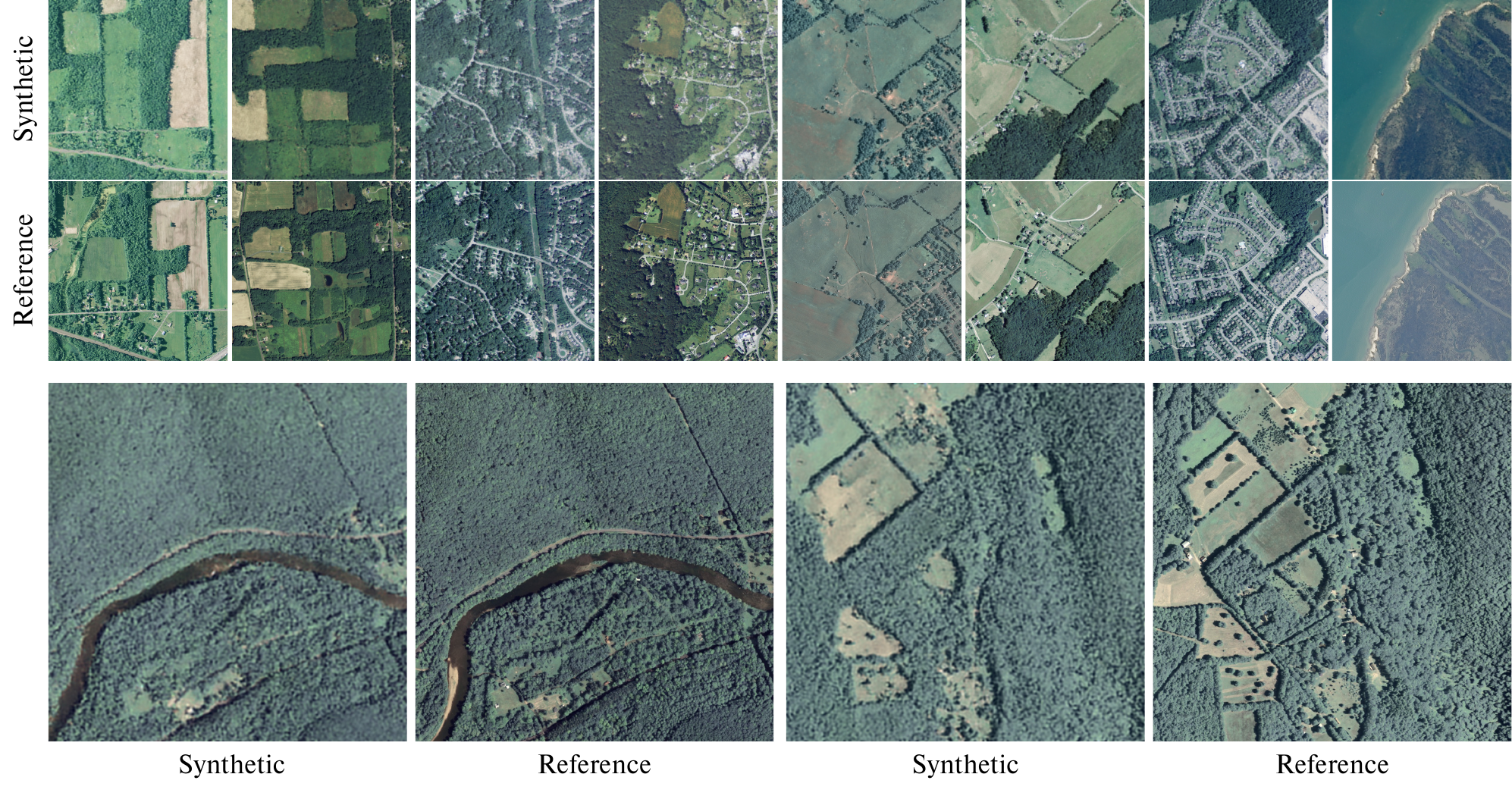}
    \caption{Satellite large ($1024 \times 1024$ and $2048 \times 2048$) images generated from \texttt{$\infty$-Brush}, and the corresponding reference real images used to generate them. Given a single embedding vector of a downsampled $256\times256$ real image, \texttt{$\infty$-Brush} can synthesize images at arbitrary resolutions and preserve global structures of the reference image.}
    \label{fig:large_image_satellite}
\end{figure}

\begin{figure}[!t]
    \centering
    \includegraphics[width=0.9\linewidth]{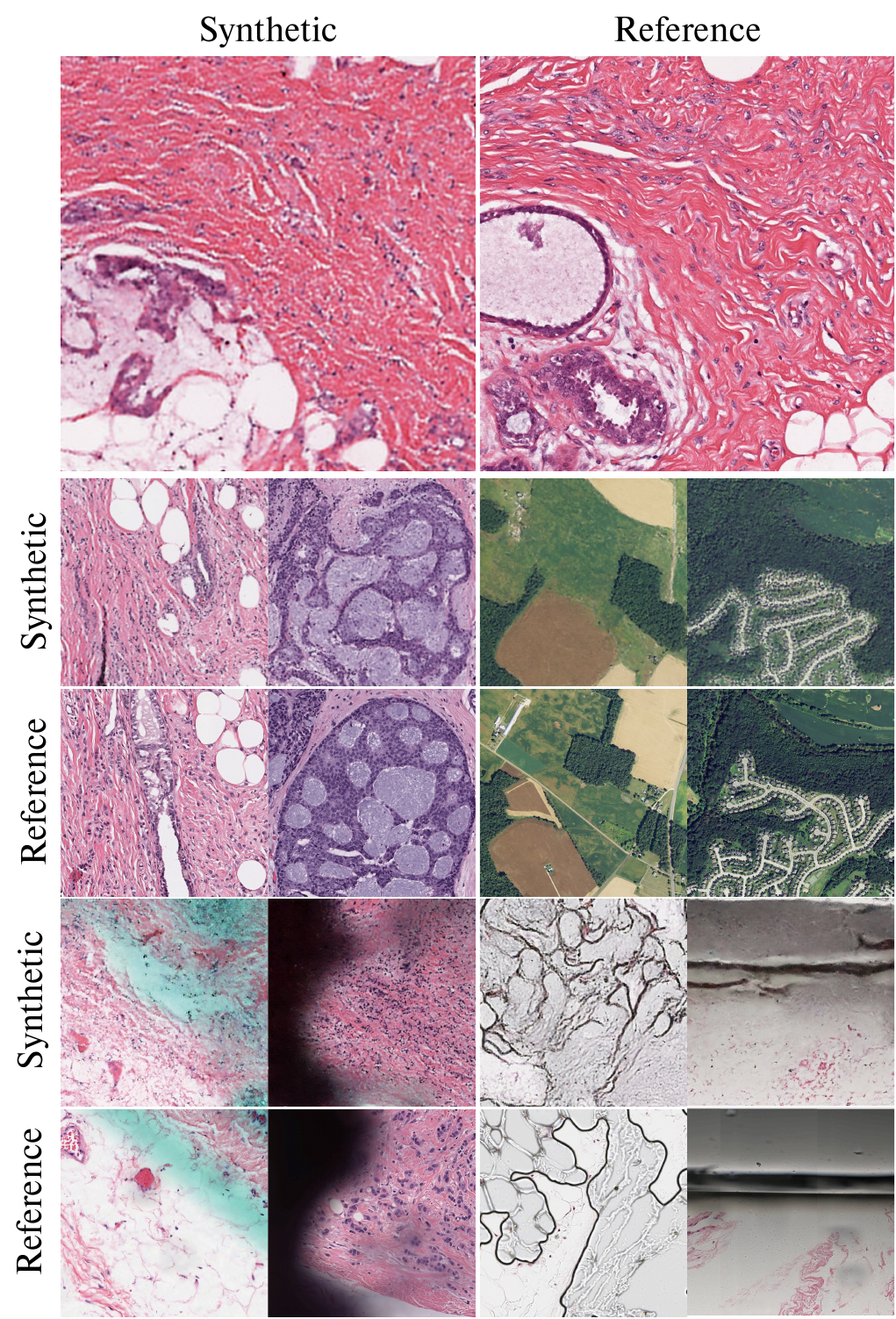}
    \caption{Uncurated ($4096 \times 4096$ and $2048 \times 2048$) images generated from \texttt{$\infty$-Brush}, and the corresponding reference real images used to generate them. Our model fails to capture spatial structure and details in specific regions of reference images (top 3 rows). In the last 2 rows, it shows that our model fails to controllably synthesize images due to bad conditioning information.}
    \label{fig:failures}
\end{figure}

\bibliographystyle{splncs04}
\bibliography{egbib}
\end{document}